\documentclass{article}

\usepackage{microtype}
\usepackage{graphicx}
\usepackage{subfigure}
\usepackage{booktabs} 
\usepackage{float}
\usepackage{lipsum}
\usepackage{multicol}
\usepackage{dblfloatfix}

\usepackage{hyperref}
\usepackage{mathtools}
\usepackage{algorithm}
\usepackage{algorithmic}

\usepackage{color-edits}
\addauthor{sw}{blue}
\addauthor{zl}{cyan}

\usepackage{amsthm}

\usepackage{amsmath,amsfonts,bm}

\def\eqref#1{equation~\ref{#1}}

\def\1{\bm{1}}

\mathchardef\mhyphen="2D

\def\Scal{{\mathcal{S}}}
\def\Dcal{{\mathcal{D}}}
\def\Acal{{\mathcal{A}}}
\def\Jcal{{\mathcal{J}}}
\def\Qcal{{\mathcal{Q}}}
\def\Ncal{{\mathcal{N}}}

\def\rmH{{\mathbf{H}}}

\DeclareMathAlphabet{\mathsfit}{\encodingdefault}{\sfdefault}{m}{sl}
\SetMathAlphabet{\mathsfit}{bold}{\encodingdefault}{\sfdefault}{bx}{n}

\newcommand{\E}{\mathbb{E}}

\newcommand{\R}{\mathbb{R}}

\newcommand{\KL}{D_{\mathrm{KL}}}

\DeclareMathOperator*{\argmax}{arg\,max}

\newtheorem{Theorem}{Theorem}
\newtheorem{Lemma}{Lemma}

\newtheorem{Proposition}{Proposition}
\newtheorem{Remark}{Remark}
\newtheorem{Assumption}{Assumption}

\usepackage[accepted]{icml2021}

\icmltitlerunning{
Constrained Variational Policy Optimization for Safe Reinforcement Learning
}

\begin{document}

\twocolumn[
\icmltitle{
Constrained Variational Policy Optimization for Safe Reinforcement Learning
}



\icmlsetsymbol{equal}{*}

\begin{icmlauthorlist}
\icmlauthor{Zuxin Liu}{cmu}
\icmlauthor{Zhepeng Cen}{cmu}
\icmlauthor{Vladislav Isenbaev}{nuro}
\icmlauthor{Wei Liu}{nuro}
\icmlauthor{Zhiwei Steven Wu}{cmu}
\icmlauthor{Bo Li}{uiuc}
\icmlauthor{Ding Zhao}{cmu}
\end{icmlauthorlist}

\icmlaffiliation{cmu}{Carnegie Mellon University}
\icmlaffiliation{nuro}{Nuro Inc.}
\icmlaffiliation{uiuc}{University of Illinois Urbana-Champaign}

\icmlcorrespondingauthor{Zuxin Liu}{zuxinl@cmu.edu}
\icmlcorrespondingauthor{Ding Zhao}{dingzhao@cmu.edu}

\icmlkeywords{Machine Learning, ICML}

\vskip 0.3in
]



\printAffiliationsAndNotice{}  

\begin{abstract}

Safe reinforcement learning (RL) aims to learn policies that satisfy certain constraints before deploying them to safety-critical applications.
Previous primal-dual style approaches suffer from instability issues and lack optimality guarantees. This paper overcomes the issues from the perspective of probabilistic inference. We introduce a novel Expectation-Maximization approach to naturally incorporate constraints during the policy learning: 1) a provable optimal non-parametric variational distribution could be computed in closed form after a convex optimization (E-step); 2) the policy parameter is improved within the trust region based on the optimal variational distribution (M-step).
The proposed algorithm decomposes the safe RL problem into a convex optimization phase and a supervised learning phase, which yields a more stable training performance.
A wide range of experiments on continuous robotic tasks shows that the proposed method achieves significantly better constraint satisfaction performance and better sample efficiency than baselines.
The code is available at \url{https://github.com/liuzuxin/cvpo-safe-rl}.

\end{abstract}

\section{Introduction}
\label{sec:intro}

The past few years have witnessed great success of reinforcement learning (RL)~\cite{mnih2013playing,silver2017mastering}. However, deploying a trained RL policy to the real world is challenging. One of the major obstacles is to ensure the learned policy satisfies safety constraints. Safe RL studies the RL problem subject to certain constraints, where the agent aims to not only maximize the task reward return, but also limit the constraint violation rate to a certain level. 
However, learning a parametrized policy that satisfies constraints is not a trivial task, especially when the policy is represented by black-box neural networks~\cite{liu2020constrained, as2022constrained, chen2021context}.

Many researchers have pointed out that prior knowledge of the environment~\cite{dalal2018safe,cheng2019end, thananjeyan2021recovery, chen2021context} and expert interventions~\cite{saunders2017trial,alshiekh2018safe,wagener2021safe} are helpful to improve safety, but the required domain knowledge on the dynamics, constraints, or the existence of an expert oracle may not always be available. 
This paper studies the safe RL problem in a more general setting: learning a safe policy purely from interacting with the environment and receiving constraint violation signals. We aim to unveil and resolve the fundamental issues for safe RL from the constrained optimization perspective.


Most constrained optimization approaches for safe RL are under the primal-dual framework, which transforms the original constrained problem into an unconstrained one by introducing the dual variables (i.e., Lagrange multipliers) to penalize constraint violations~\cite{chow2017risk,liang2018accelerated, tessler2018reward, bohez2019value, ray2019benchmarking}. 
However, the primal-dual iterative optimization may run into numerical \textit{instability} issues and lacks \textit{optimality} guarantee for each policy iteration~\cite{chow2018lyapunov, stooke2020responsive}. 
The instability usually comes from imbalanced learning rates of the primal and dual problem, and the optimality term means both \textit{feasibility} (constraint satisfaction) and reward maximization.
Another line of work approximate the constrained optimization problem with low-order Taylor expansions such that the dual variables could be solved efficiently~\cite{achiam2017constrained, yang2020projection, zhang2020first} , but the induced approximations errors may yield poor constraint satisfaction performance in practice~\cite{ray2019benchmarking}.
In addition, these policy-gradient algorithms are on-policy by design, and extending them to a more sample efficient off-policy setting is non-trivial. Note that in safe RL context, sample efficiency means using both minimum constraint violation costs and minimum interaction samples to achieve the same level of rewards.
As a result, a constrained RL optimization method that is 1) sample efficient, 2) stable and has 3) performance guarantees is absent in the literature.

To bridge the gap, this paper proposes the Constrained Variational Policy Optimization (CVPO) algorithm from the probabilistic inference view. CVPO transforms the safe RL problem to a convex optimization phase with optimality guarantees and a supervised learning phase with policy improvement bound and the robustness guarantee to recover to the feasible region, which ensure stable training performance. The off-policy implementation also gives rise to high sample-efficiency in practice.
The main contributions of this work are summarized as follows:
\begin{enumerate}
    \item To our best knowledge, this is the first work that formulates safe RL as a \textit{probabilistic inference} problem, which enables us to leverage the rich toolbox of inference techniques to solve the safe RL problem. These techniques are used to overcome many drawbacks of previous primal-dual policy optimization fashion, such as unstable training and lack of optimality guarantee.
    
    \item We propose a novel two-step  algorithm in an Expectation Maximization (EM) style to naturally incorporate safety constraints during the policy training. An optimal and feasible non-parametric variational distribution is solved \textit{analytically} during the E-step, and then the parametrized policy is trained via a \textit{supervised learning} fashion during the M-step, which allows us to improve the policy with off-policy data and increase the sample efficiency.

    \item We show the closed-form variational distribution in the E-step could be \textbf{computed efficiently} with \textbf{provable optimality guarantee}. The efficiency arises from the strict convexity of the dual problem in most cases, which ensures the uniqueness and optimality of the solution, and we did not find similar claims and proofs in the literature. Furthermore, the trust-region regularized policy improvement during the M-step gives us a worst-case constraint violation bound and robustness guarantee against worst-case training iterations.
    
    \item We evaluate CVPO on a series of continuous control tasks. The empirical experiments demonstrate the effectiveness of the proposed method -- more stable training, better constraint satisfaction, and up to 1000 times better sample efficiency than on-policy baselines.
\end{enumerate}

\section{Preliminary and Related Work}
\subsection{Constrained Markov Decision Processes}
\label{sec:cmdp}
Constrained Markov Decision Processes (CMDPs) provide a mathematical framework to describe the safe RL problem~\citep{altman1998constrained}, where the agents are enforced with restrictions on auxiliary safety constraint violation costs.
A CMDP is defined by a tuple $(\Scal, \Acal, P, r, \gamma, \rho_0, C)$, where $\Scal$ is the state space, $\Acal$ is the action space, $P:\Scal \times \Acal \times \Scal \xrightarrow{} [0, 1]$ is the transition kernel that specifies the transition probability $p(s_{t+1} | s_t, a_t)$ from state $s_t$ to $s_{t+1}$ under the action $a_t$, $r:\Scal \times \Acal \xrightarrow{} \mathbb{R}$ is the reward function, $\gamma \xrightarrow{} [0, 1)$ is the discount factor, and $\rho_0: \Scal \xrightarrow[]{} [0,1]$ is the distribution over the initial states. The last element $C$ is a set of costs $\{c_i:\Scal \times \Acal \xrightarrow{} \mathbb{R}_{\ge 0}, i=1,2,...,m\}$ for violating $m$ constraints, which is the major difference between CMDP and traditional Markov Decision Process (MDP). Depending on the application, the cost $c_i \in C$ has different representations and physical meanings. For instance, it could be an indicator cost signal for being in an unsafe set of states and actions, or a continuous function of the distance w.r.t the constraint boundary.
For simplicity of notation, we consider one universal constraint function $c \in C:\Scal \times \Acal \xrightarrow{} \mathbb{R}_{\ge 0}$ exists in the CMDP to characterize the corresponding constrained RL problem, which resembles the notation of reward function. All the definitions and theorems in this paper could be extended to multiple constraints as well. 

Let $\pi(a|s)$ denote the policy, and $\tau = \{s_0, a_0, ..., \}$ denote the trajectory. We use shorthand $r_t=r(s_t, a_t)$ and $c_t=c_t(s_t, a_t)$ for simplicity. The discounted expect return of the reward under the policy $\pi$ is $J_r(\pi) = \mathbb{E}_{\tau \sim \pi}[ \sum_{t=0}^\infty \gamma^t r_t ]$, and similarly the discounted expected return of the cost is $J_c(\pi) = \mathbb{E}_{\tau \sim \pi}[ \sum_{t=0}^\infty \gamma^t c_t ]$, where the initial state $s_0 \sim \rho_0$.
The objective of a safe RL problem is to find the policy $\pi^*$ that maximizes the expected cumulative rewards while limiting the costs incurred from constraint violations to a threshold $\epsilon_1 \in [0, +\infty)$:
\begin{equation}
   \pi^* = \arg\max_{\pi}  J_r(\pi), \quad s.t. \quad J_c(\pi)\leq \epsilon_1.
   \label{eq:cmdp_objective}
\end{equation}

\subsection{Related Work}
\label{sec:related}
\textbf{Constrained RL optimization.}
Primal-dual approach is most commonly used in solving safe RL problems~\cite{ding2020natural, bohez2019value}. Generally, the primal-dual style algorithms alternate between optimizing the policy parameters and updating the dual variables, which are usually performed via gradient descent~\cite{tessler2018reward, liang2018accelerated, zhang2020first}. \citet{stooke2020responsive} view the dual problem as a control system and propose a PID control method to update the Lagrange multipliers in a more stable way.
Though the primal-dual framework is intuitive, the trained policy
makes little safety guarantee with respect to both the converged policy and the behavior policy during each training iteration~\cite{chow2018lyapunov, xu2021crpo}. 
Several works introduce a KL-regularized policy improvement mechanism and provide the worse-case performance bound~\cite{achiam2017constrained, yang2020projection}, however, the quadratic approximation of the original problem usually lead to high cost~\cite{ray2019benchmarking}. 
In addition, the primal-dual approaches heavily rely on an accurate on-policy value estimation of constraints, so applying them to off-policy settings is not easy: one needs to backpropagate  gradients from multiple Q-value functions to the policy network, which may cause instability issues. 


\textbf{RL as inference.} Formulating RL as inference has been extensively studied recently~\cite{ levine2018reinforcement}.
\citet{haarnoja2017reinforcement, haarnoja2018soft} perform exact inference over the probabilistic graphical model of RL via message passing. \citet{abdolmaleki2018maximum, abdolmaleki2018relative} propose the Maximum a posteriori Policy Optimization (MPO) method to improve the policy with off-policy data in an Expectation-Maximization fashion, which can achieve comparable performance and better sample efficiency than on-policy methods. \citet{song2019v} extend MPO to on-policy setting, and \citet{fellows2019virel} propose a variational inference (VI) framework for RL.
While we believe the flourishing development of powerful RL as inference methods could provide a fresh view over the safe RL domain, however, a theoretical exploration of the connection between them has so far been lacking.

\section{Constrained Variational Policy Optimization (CVPO)}
We observe that under mild assumptions, safe RL could be viewed as a probabilistic inference problem, which yields CVPO --- a generic approach to incorporate safety constraints in the inference step. We will detail our method in this section and show how CVPO inherits many theoretical benefits from both the RL as inference and the constrained RL domain.
\subsection{Constrained RL as Inference}
\label{sec:safe_inference}
Before presenting the inference view, we first introduce the standard primal-dual perspective to solve CMDP, which transforms the objective (\ref{eq:cmdp_objective}) into a min-max optimization by introducing the Lagrange multiplier $\lambda$: 
\begin{align}
    (\pi^*, \lambda^*) = \arg \min_{\lambda \geq 0}\max_\pi J_r(\pi) - \lambda (J_c(\pi) -\epsilon_1 ).
\end{align}
The core principle of primal-dual approaches is to solve the min-max problem iteratively. However, the optimal dual variable $\lambda=+\infty$ when $J_c(\pi) > \epsilon_1$ and $\lambda=0$ when $J_c(\pi) < \epsilon_1$, so selecting a proper learning rate for $\lambda$ is critical. Approximately solving the minimization also leads to suboptimal dual variables for each iteration. In addition, the non-stationary cost penalty term involving $\lambda$ will make the policy gradient step in the primal problem hard to optimize, just as shown in the upper diagram of Fig.~\ref{fig:overall}.

\begin{figure}[!h]
\centering     
\includegraphics[width=0.9\linewidth]{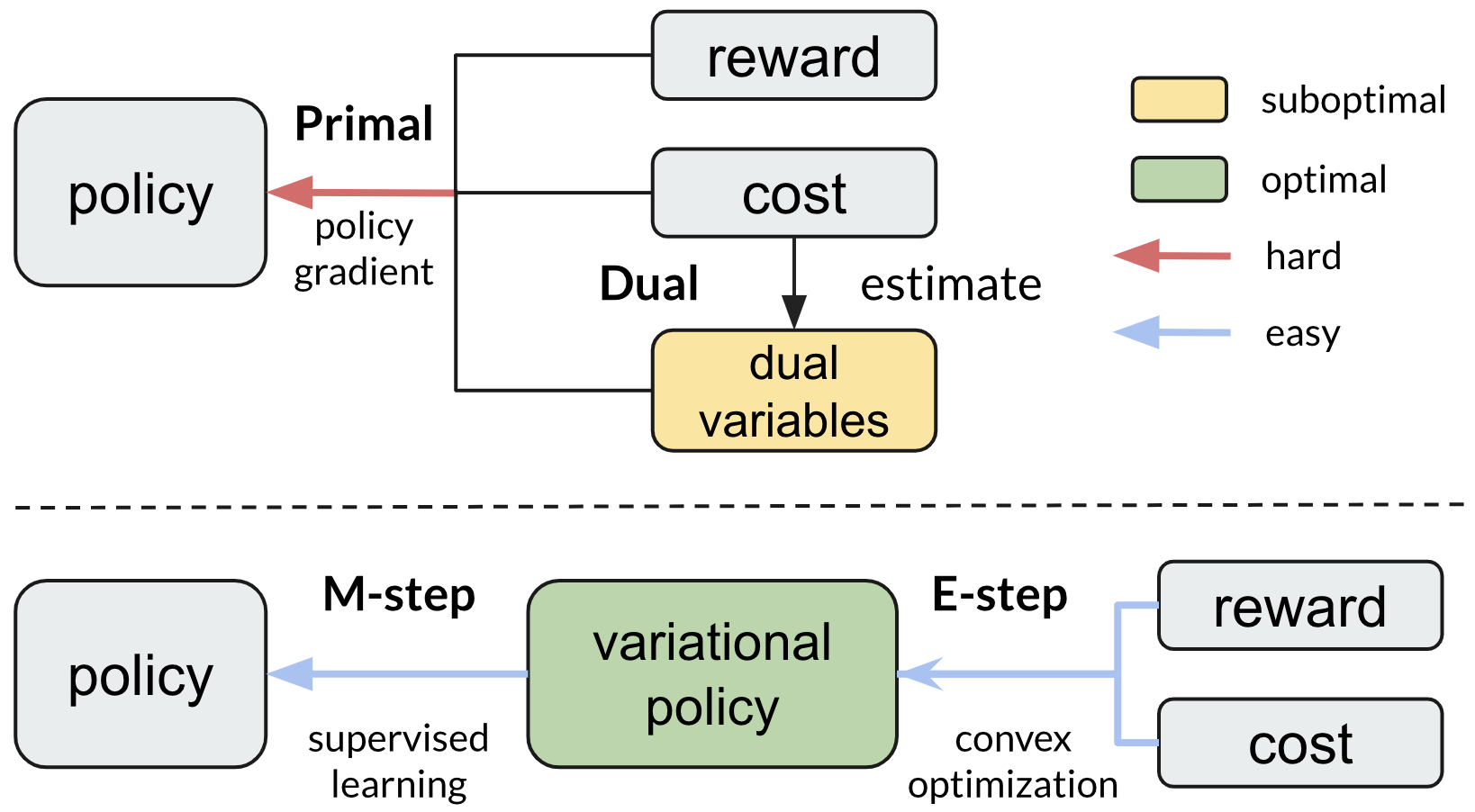}
\caption{Primal-dual view (upper) and inference view (lower).}
\label{fig:overall}
\end{figure}

To tackle the above problems, we view the safe RL problem from the probabilistic inference perspective --
inferring safe actions that result in “observed” high reward in states. This is done by introducing an optimality variable $O$ to represent the event of maximizing the reward.
Following similar probabilistic graphical models and notations in~\cite{levine2018reinforcement}, 
we consider an infinite discounted reward formulation, where the likelihood of being optimal given a trajectory is proportional to the exponential of the discounted cumulative reward:\linebreak  $p(O=1\mid \tau) \propto \exp(\sum_t \gamma^t r_t/\alpha)$, where $\alpha$ is a temperature parameter. Denote the probability of a trajectory $\tau$ under the policy $\pi$ as $p_\pi(\tau)$, then the lower bound of the log-likelihood of optimality under the policy $\pi$ is (see Appendix~\ref{app:elbo} for proof):
\begin{align}
    &\log  p_\pi  (O=1) = \log \int  p(O=1 \mid \tau)p_\pi(\tau) d\tau \\
    & \geq \E_{\tau \sim q} [\sum_{t=0}^\infty \gamma^t r_t] - \alpha \KL(q(\tau)  \Vert p_\pi(\tau)) = \Jcal(q,\pi) \label{eq:elbo}
\end{align}
where $q(\tau)$ is an auxiliary trajectory distribution and $\Jcal(q,\pi)$ is the evidence lower bound (ELBO).
To ensure safety, we limit the choices of $q(\tau)$ within a feasible distribution family that is subject to constraints. 
Recall that $c_t = c(s_t,a_t)$ is the cost for constraint violations. We define the \textbf{feasible distribution family} in terms of the threshold $\epsilon_1$ as
$$\Pi_{\Qcal}^{\epsilon_1}\coloneqq \{q(a|s) : \E_{\tau \sim q}[\sum_{t=0}^\infty \gamma^t c_t] < \epsilon_1, a \in \Acal, s \in \Scal\},$$ which is a set of all the state-conditioned action distributions that satisfy the safety constraint.
Afterwards, by factorizing the trajectory distributions 
$$
    q(\tau)  = p(s_0)\prod_{t\geq0}p(s_{t+1}|s_t, a_t) q(a_t|s_t), \forall q \in \Pi_\Qcal^{\epsilon_1},
$$
$$
    p_{\pi_\theta}(\tau)  = p(s_0)\prod_{t\geq0}p(s_{t+1}|s_t, a_t) \pi_\theta(a_t|s_t) p(\theta),
$$
where $\theta \in \Theta$ is the policy parameters, and $p(\theta)$ is a prior distribution, we obtain the following ELBO over the feasible state-conditioned action distribution $q(a|s)$ by cancelling the transitions:
\begin{equation}
\begin{split}
    \Jcal(q,\theta) =&
    \E_{\tau \sim q} \left[ \sum_{t=0}^\infty \left(\gamma^t r_t - \alpha\KL(q(\cdot| s_t)  \Vert \pi_\theta(\cdot|s_t))\right)\right] \\
    & + \log p(\theta), \quad \forall q(a|s_t) \in \Pi_\Qcal^{\epsilon_1}.
\end{split}
\label{eq:elbo_new}
\end{equation}
Optimizing the new lower bound $\Jcal(q,\theta)$ w.r.t $q$ within the feasible distribution space $\Pi_\Qcal^{\epsilon_1}$ (E-step) and $\pi$ within the parameter space $\Theta$ (M-step) iteratively via an Expectation-Maximization (EM) fashion yields the Constrained Variational Policy Optimization (CVPO) algorithm.
We could interpret the inference formulation as answering ``given future success in maximizing task rewards, what are the \textbf{feasible} actions most likely to have been taken?", instead of ``what are the actions that could maximize task rewards while satisfying the constraints?" in the primal-dual formulation.
The decoupling choice by separating the E-step and M-step provides more flexibility to optimizing the control policies and selecting their representations.
Similar idea is also adopted in previous works for standard RL setting~\cite{abdolmaleki2018maximum, abdolmaleki2018relative}.

Viewing safe RL as a variational inference problem has many benefits. As shown in the lower diagram of Fig. \ref{fig:overall}, we break the direct link between the inaccurate dual variable optimization and the difficult policy improvement and introduce a variational distribution in the middle to bridges the two steps, such that the policy improvement could be done via a much easier supervised learning fashion. The \textbf{key challenges} arise from the E-step because of the limitation of the variational distribution within a constrained set. However, we observe that the constrained $q$ could be solved analytically, efficiently, with optimality and feasibility guarantee through convex optimization, as we show next.


\subsection{Constrained E-step}
\label{sec:estep}
The objective of this step is to find the optimal variational distribution $q \in \Pi_{\Qcal}^{\epsilon_1}$ to improve the return of task reward, while satisfying the safety constraint.
At the $i\mhyphen th$ iteration, we resort to perform a partial constrained E-step to maximize $\Jcal(q,\theta)$ with respect to $q$ by fixing the policy parameters $\theta = \theta_i$. We set the initial value of $q = \pi_{\theta_i}$ such that the return of task reward $Q_r^q(s,a)$ and cost $Q_c^q(s,a)$ could be estimated by: \linebreak $Q_r^q(s,a) = Q_r^{\pi_{\theta_i}}(s,a) = \E_{\tau\sim\pi_{\theta_i}, s_0=s,a_0=a}\Big[\sum_{t=0}^\infty \gamma^t r_t\Big]$,
\linebreak $Q_c^q(s,a) = Q_c^{\pi_{\theta_i}}(s,a) = \E_{\tau\sim\pi_{\theta_i}, s_0=s,a_0=a}\Big[\sum_{t=0}^\infty \gamma^t c_t\Big]$, where the trajectory $\tau$ could be sampled from the replay buffer and thus the critics could be updated in an off-policy fashion. Given $\pi_{\theta_i}$, $Q_r^{\pi_{\theta_i}}(s,a)$ and $Q_c^{\pi_{\theta_i}}(s,a)$, we further optimize $q(\cdot|s)$ by the following KL regularized objective:
\begin{equation}
\begin{split}
    \Bar{\Jcal}(q) &= \E_{\rho_q}\Big[\E_{q(\cdot|s)} \big[ Q_r^{\pi_{\theta_i}}(s,a)\big] - \alpha \KL(q \Vert \pi_{\theta_i}) \Big], \\
    & s.t. \quad \E_{\rho_q}\Big[\E_{q(\cdot|s)} \big[ Q_c^{\pi_{\theta_i}}(s,a)\big] \Big] \leq \epsilon_1
\end{split}
    \label{eq:estep}
\end{equation}
where $\rho_q(s)$ is the stationary state distribution induced by $q(a|s)$ and $\rho_0$. The constraint ensures the optimized distribution is within the feasible set $\Pi_\Qcal^{\epsilon_1}$.
Solving the E-step (\ref{eq:estep}) could be regarded as a KL-regularized constrained optimization problem. However, since the expected reward return term $\E_{q(\cdot|s)} \big[ Q_r^{\pi_{\theta_i}}(s,a)\big]$ could be on an arbitrary scale, it is hard to choose a proper penalty coefficient $\alpha$ of the KL regularizer for different CMDP settings. Therefore, we impose a hard constraint $\epsilon_2$ on the KL divergence between the non-parametric distribution $q(a|s)$ that to be optimized and the parametrized policy $\pi_{\theta_i}(a|s)$. Then the E-step yields the following constrained optimization:
\begin{equation}
\begin{split}
   \max_{q} & \quad \E_{\rho_q}\Big[ \int q(a|s) Q_r^{\pi_{\theta_i}}(s,a) da \Big]\\
    s.t. & \quad \E_{\rho_q}\Big[\int q(a|s) Q_c^{\pi_{\theta_i}}(s,a) da \Big] \leq \epsilon_1 \\
    & \quad \E_{\rho_q}\Big[\KL(q(a|s)\Vert \pi_{\theta_i}) \Big] \leq \epsilon_2 \\
    & \quad \int q(a | s) da = 1, \quad \forall s \sim \rho_q
\end{split}
\label{eq:estep_objective}
\end{equation}
where we have three constraints for this optimization problem in total. The \textbf{first safety constraint} in terms of $\epsilon_1$ is to ensure the optimal non-parametric distribution belongs to the feasible set such that the safety constraints of CMDP could be satisfied. The \textbf{second regularization constraint} in terms of the KL threshold $\epsilon_2$ aims to restrict the optimized variational distribution $q$ within a trust region of the old policy distribution. 
The \textbf{last equality constraint} is to make sure the solved $q$ is a valid action distribution across all the states. Intuitively, in E-step, we aim to find the optimal variational distribution that 1) maximizes the task rewards, 2) belongs to the feasible distribution family $\Pi_\Qcal^{\epsilon_1}$, and 3) stays within the trust region of the old policy. 

Before solving the constrained optimization problem~(\ref{eq:estep_objective}), we need to specify the representation of the variational distribution $q$. 
The inference formulation of the safe RL problem gives us the flexibility to choose its representation freely.
Note that if we use a parametric representation of $q(a|s)$, then the E-step is similar to the policy updating in CPO, where we could regard optimizing $q$ as updating the policy parameters from $\theta_i$ to $\theta_{i+1}$. 
As such, the M-step is no longer required. However, as shown in CPO, the constrained optimization problem with a parametrized $q$ is generally intractable, so approximations over the parameter space are usually required, which may lead to poor constraint satisfaction~\cite{ray2019benchmarking}.

To avoid the performance degradation induced by approximation errors, we choose a non-parametric form for $q(a|s)$. 
Namely, given a state $s$, we use $|\Acal|$ variables for $q(\cdot|s)$ if the action space is finite. 
Otherwise, we sample $K$ particles within the action space to represent the variational distribution for the continuous action space. 
By constructing a non-parametric form of $q$, we could see that problem (\ref{eq:estep_objective}) is a convex problem since the objective is linear and all the constraints are convex. Moreover, we could obtain the optimal (and mostly unique) solution in an analytical form (\ref{eq:optimalq_theorem}) after solving a convex dual problem. We show the strong duality guarantee under mild assumptions (see Appendix~\ref{app:dual_function}):
\begin{Assumption}
(Slater's condition). There exists a feasible distribution $\Bar{q} \in \Pi_\Qcal^{\epsilon_1}$ within the trust region of the old policy $\pi_{\theta_i}$: $\KL(\Bar{q}\Vert \pi_{\theta_i}) < \epsilon_2$.
\label{assumption:slater}
\end{Assumption}
\begin{Theorem}
  If assumption \ref{assumption:slater} holds, then the optimal variational distribution within $\Pi_\Qcal^{\epsilon_1}$ for problem (\ref{eq:estep_objective}) has the form:
 \begin{align}
    q^*_i(a|s) = \frac{\pi_{\theta_i}(a|s)}{Z(s)} \exp\left(\frac{Q_r^{\theta_i}(s,a) - \lambda^* Q_c^{\theta_i}(s,a)}{\eta^*}\right)
    \label{eq:optimalq_theorem}
\end{align}
where $Z(s)$ is a constant normalizer to make sure $q^*$ is a valid distribution, and the dual variables $\eta^*$ and $\lambda^*$ are the solutions of the following convex optimization problem:
\begin{equation}
\begin{split}
    & \min_{\lambda, \eta \geq 0} \quad  g(\eta, \lambda) = \lambda \epsilon_1 + \eta \epsilon_2 + \\
    & \eta \E_{\rho_q} \left[ \log \E_{\pi_{\theta_i}} \Big[ \exp\left(\frac{Q_r^{\theta_i}(s,a) - \lambda Q_c^{\theta_i}(s,a)}{\eta}\right) \Big] \right].
\end{split}
\label{eq:dual_opt}
\end{equation}
\label{theorem:1}
\end{Theorem}
Theorem \ref{theorem:1} suggests the non-parametric variational distribution could be easily solved in close-form with optimality guarantee, since Assumption \ref{assumption:slater} ensures zero duality gap.
Eq. (\ref{eq:optimalq_theorem}) indicates that the optimal $q$ is re-weighted based on the old policy $\pi_{\theta_i}$, where the weights are controlled by $Q_r^{\theta_i}(s,a), Q_c^{\theta_i}(s,a), \eta, \lambda$. Note that the $Q_r^{\theta_i}(s,a)$ and $Q_c^{\theta_i}(s,a)$ are viewed as constants here as discussed in section~\ref{sec:safe_inference}.
We can see that higher weights are given to the actions that have higher future task rewards and lower safety costs, where the weight between them is balanced by $\lambda$. 
Intuitively, the dual variable $\eta$ serves as a temperature to control the flatness of the weights, such that the updated variational distribution would not be far away from the old policy or collapse to one action quickly. Higher $\eta$ enforces stronger restrictions on the flatness, which makes sense since we limit the solution within a trust region. 
One exciting property of Theorem \ref{theorem:1} is that we prove the strong convexity conditions of the dual problem, which guarantees the optimality and uniqueness of the solution, as shown below:
\begin{Theorem}
The multivariate dual function $g(\eta, \lambda)$ in (\ref{eq:dual_opt}) is convex on $\R^2_{\geq0}$. It is strictly convex and has a unique optimal solution when (1) $Q_r^{\theta_i}(s,\cdot), Q_c^{\theta_i}(s,\cdot)$ are not constant functions; (2) $\forall C\in\mathbb{R}, \exists a_0, s.t.  Q_r^{\theta_i}(s,a_0)\neq C\cdot Q_c^{\theta_i}(s,a_0)$; and (3) $\lambda < +\infty$. When the Slater condition in Assumption \ref{assumption:slater} holds, at least one optimal solution exists.
\label{theorem:dual_convex}
\end{Theorem}
Proof and discussions are in Appendix \ref{app:convex}. 
Each condition has corresponding meaning, as we explain in the following remarks.

\begin{Remark}
For the first condition, if $Q_r(s,\cdot)$ is a constant function, then the objective in (\ref{eq:estep_objective}) will always be a constant --- the optimization becomes meaningless since all the distributions that within the trust region should be the same; if $Q_c(s,\cdot)$ is a constant function, then the safety constraint will be inactive, since no policy could change the feasibility status. 
In addition, if $Q_c(s,\cdot)$ is a constant, the gradient of the safety constraint dual variable $\lambda$ will either be a positive constant when the problem is feasible (tends to make $\lambda\xrightarrow[]{} 0$) or a negative constant when the problem is infeasible (tends to make $\lambda\xrightarrow[]{} +\infty$). 
\end{Remark}
\begin{Remark}
The second condition indicates that if the task reward $Q_r$ value is proportional to the cost $Q_c$ value everywhere, then multiple optimal solutions may exist on the KL constraint boundary or the safety constraint boundary. 
Intuitively, the problem becomes a bounded convex optimization --- maximizing the risks $\E_q[Q_c(s,a)]$ within the trust region defined by $\epsilon_2$ and a ball defined by $\epsilon_1$. 
Though the solutions may not be unique, the optimality could still be guaranteed as long as the Slater's condition \ref{assumption:slater} holds.
\end{Remark}
\begin{Remark}
The third condition is equivalent to the violation of our Slater's condition assumption \ref{assumption:slater} --- no distribution exists within the trust region that satisfies the safety constraints. In that sense, the gradient of the dual variable $\lambda$ will always be negative, and thus $\lambda \xrightarrow[]{} +\infty$, which tends to impose huge penalties on safety critics $Q_c$. 
In practice, this rarely happens since we could select a large KL threshold for E-step and choose a much smaller trust region size for M-step, such that $n$-step robustness could be guaranteed and the Slater's condition could be easily met, as we show in Appendix \ref{app:n_step_robustness}.
\end{Remark}

From Theorem \ref{theorem:dual_convex} and above remarks, we can see the conditions for strict convexity are easy to be satisfied: the Q value functions should not be constants and the reward return will not be proportional to the cost return in most safe RL settings. 
The strict convexity analysis guarantees the uniqueness of the optimal solution in most situations and provides some other theoretical implications of the E-M procedure, such as the convergence to a stationary point \cite{sriperumbudur2009integral}.
Furthermore, the Slater condition indicates that $\lambda < +\infty$ and the existence of an optimal solution, which ensures the dual problem could be solved efficiently via any convex optimization tools. 
We believe this finding is original and non-trivial, which also lays the theoretical foundation of the outstanding empirical performance of our method, as we show in Sec.~\ref{sec:experiment}. 


\subsection{M-step}
\label{sec:mstep}
After E-step, we obtain an optimal feasible variational distribution $q^*_i(\cdot|s)$ for each state by solving (\ref{eq:estep_objective}). In M-step, we aim to improve the ELBO (\ref{eq:elbo_new}) w.r.t the policy parameter $\theta$.
By dropping the terms in Eq. (\ref{eq:elbo_new}) that are independent of $\theta$, we obtain the following M-step objective:
\begin{equation}
\begin{split}
    \Bar{\Jcal}(\theta) = \E_{\rho_q}\Big[\alpha \E_{q^*_i(\cdot|s)} \big[ \log \pi_\theta (a|s)\big]\Big] + \log p(\theta)
\end{split}
\label{eq:mstep}
\end{equation}
where $\alpha$ is a temperature parameter to balance the weight between likelihood and prior, and $q^*_i(\cdot|s)$ serve as the weights for the samples $\pi_\theta$. Note that the objective could be viewed as a weighted maximum a-posteriori problem, which is similar to MPO~\cite{abdolmaleki2018maximum,abdolmaleki2018relative}. Since optimizing $\Bar{\Jcal}(\theta)$ is a supervised learning problem, we could use any prior $p(\theta)$ to regularize the policy and use any optimization methods to solve~(\ref{eq:mstep}), depending on the policy parametrization. 
Particularly, we adopt a Gaussian prior around the old policy parameter $\theta_i$ in this paper: $\theta \sim \Ncal(\theta_i, \frac{ F_{\theta_i}}{\alpha\beta})$, where $F_{\theta_i}$ is the Fisher information matrix and $\beta$ is a positive constant. With this Gaussian prior, we obtain the following generalized M-step (see Appendix~\ref{app:mstep} for detailed proof):
\begin{equation}
\begin{split}
    \max_\theta \E_{\rho_q}\Big[ \E_{q^*_i(\cdot|s)} \big[ \log \pi_\theta (a|s)\big] - \beta \KL(\pi_{\theta_i}\Vert \pi_{\theta}) \Big].
\end{split}
\end{equation}
Similar to the E-step, we convert the soft KL regularizer to a hard KL constraint to deal with different objective scales:
\begin{equation}
\begin{split}
    &\max_\theta \quad \E_{\rho_q}\Big[ \E_{q^*_i(\cdot|s)} \big[ \log \pi_\theta (a|s)\big]\Big] \\
    & s.t. \quad \E_{\rho_q}\big[ \KL(\pi_{\theta_i}(a|s)\Vert \pi_{\theta}(a|s)) \big] \leq \epsilon.
\end{split}
\label{eq:m_step_kl}
\end{equation}
The regularizer is important to improve the updating robustness and prevent the policy from overfitting, as we will show in the experiment section~\ref{sec:ablation_study}.

\subsection{Theoretical Analysis}
\label{sec:theory_bound}
CVPO updates the policy by maximizing KL-regularized objective functions in an EM manner, which brings it two advantages over primal-dual methods -- the ELBO improvement guarantee and the worst-case constraint satisfaction bound (see Appendix~\ref{app:ELBO_improve} \&~\ref{app:worst_case} for details).

\begin{Proposition}
  Suppose $\pi_{\theta_{i-1}}, \pi_{\theta_{i}}$ satisfy the Slater's condition, then the ELBO in Eq. (\ref{eq:elbo_new}) is guaranteed to be non-decreasing: $J(q_i, \theta_{i+1})\geq J(q_{i-1}, \theta_{i})$.
\label{Proposition:improvement}
\end{Proposition}

Proposition \ref{Proposition:improvement} provides the policy improvement guarantee for the reward. Regarding the constraint violation cost, we have the following bound holds for CVPO:
\begin{Proposition}
  Suppose $\pi_{\theta_{i}} \in \Pi_\Qcal^{\epsilon_1}$. $\pi_{\theta_{i+1}}$ and $\pi_{\theta_{i}}$ are related by the M-step (\ref{eq:m_step_kl}), then we have the upper bound:
\begin{equation}
    J_c(\pi_{\theta_{i+1}}) \leq \epsilon_1 + \frac{[(1-\gamma)+\sqrt{2\epsilon}\gamma]\delta_c^{\pi_{\theta_{i+1}}}}{(1-\gamma)^2}
\end{equation}
where $\delta_{c}^{\pi_{\theta_{i+1}}} = \max_s \vert \E_{a\sim \pi_{\theta_{i+1}}}[A_c^{\theta_i}(s,a)]\vert$, $A_c^{\theta_i}(s,a)$ is the cost advantage function of $\pi_{\theta_i}$.
\label{Proposition:cost_bound}
\end{Proposition}
We can see the worst-case episodic constraint violation upper bound for $\pi_{\theta_{i+1}}$ is related to the trust region size $\epsilon$ and the worst-case approximation error $\delta_{c}^{{\theta_{i+1}}}$. Though we inherit a similar bound as CPO, our method ensures the optimality of each update and could be done in a more sample-efficient off-policy fashion. In addition, we observe that by selecting proper KL constraints $\epsilon$ in the M-step, we have the following policy improvement robustness property:
\begin{Proposition}
  Suppose $\pi_{\theta_{i}} \in \Pi_\Qcal^{\epsilon_1}$. $\pi_{\theta_{i+1}}$ and $\pi_{\theta_{i}}$ are related by the M-step. If $\epsilon < \epsilon_2$, where $\epsilon, \epsilon_2$ are the KL threshold in M-step and E-step respectively, then the variational distribution $q_{i+1}^*$ in the next iteration is guaranteed to be feasible and optimal.
\label{Proposition:slater_gurantee}
\end{Proposition}
Proposition \ref{Proposition:slater_gurantee} provides us with an interesting perspective and a theoretical guarantee of the policy improvement robustness -- no matter how bad the M-step update in one iteration is, we are still able to recover to an optimal policy within the feasible region. Furthermore, We find that under the Gaussian policy assumption, multiple steps robustness could also be achieved (see Appendix \ref{app:n_step_robustness} for proof and figure illustrations). The monotonic reward improvement guarantee, the worst-case cost bound and the policy updating robustness together ensure the \textbf{training stability} of CVPO.

\subsection{Practical Implementation}
\label{sec:implementation}
Another advantage of the proposed scheme is that the framework can be easily applied to a more sample efficient off-policy setting, which is important for safe RL because worse samples efficiency usually indicates more constraint violations. 
Practically, the stationary state distribution $\rho_q$ could be approximated by the samples from the replay buffer, which yields the off-policy version of CVPO. 
Algo.~\ref{algo:cvpo} highlights the key steps of one training epoch for continuous action space. 
To solve tasks with discrete action space, we only need to replace summation with integration over action space. 
We also decompose the KL constraints into separate constraints on mean and covariance during the M-step (\ref{eq:m_step_kl}) to achieve better exploration-exploitation trade-off~\cite{abdolmaleki2018maximum}.
For more implementation details and the full algorithm, please refer to Appendix \ref{app:implementation}.
\begin{algorithm}[h]
\caption{CVPO Training for One Epoch}
{\bfseries Input:} \raggedright batch size $B$, particle size $K$, policy parameter $\theta_i$ \par
{\bfseries Output:} \raggedright Updated policy parameter $\theta_{i+1}$ \par
\begin{algorithmic}[1] 
\STATE Sample $B$ transitions from replay buffer
\STATE $\triangleright$ \textit{E-step begins}
\STATE Update $Q_{r}^{\theta_i}, Q_{c}^{\theta_i}$ via Bellman backup
\FOR{$b=1,..., B$}
\STATE Sample $K$ actions $\{a_1,...,a_K\}$ for $s_b$
\STATE Compute $\{Q_{r}^{\theta_i}(s_b,a_k), Q_{c}^{\theta_i}(s_b,a_k); k=1,...,K\}$.
\ENDFOR
\STATE Compute optimal dual variables $\eta^*, \lambda^*$ by solving the convex optimization problem (\ref{eq:dual_opt})
\STATE Compute the optimal variational distribution for each state $\{q^*(\cdot|s_b); b=1,...,B\}$ by Eq. (\ref{eq:optimalq_theorem})
\STATE $\triangleright$ \textit{M-step begins}
\STATE Update policy from $\pi_{\theta_i}$ to $\pi_{\theta_{i+1}}$ via supervised learning objective (\ref{eq:m_step_kl})
\end{algorithmic} \label{algo:cvpo}
\end{algorithm}

\section{Experiment}
\label{sec:experiment}

Motivated by previous works~\cite{ray2019benchmarking, achiam2017constrained, zhang2020first, stooke2020responsive, gronauer2022bullet}, we designed 5 robotic control tasks with different difficulty levels to show the effectiveness of our method.
Detail description of the task environments and method implementations could be found in Appendix \ref{app:implementation_detail}.
The code is also available at \url{https://github.com/liuzuxin/cvpo-safe-rl}.

\subsection{Tasks}
\textbf{Circle Task.} A car robot is expected to move on a circle in clock-wise direction. The agent will receive higher rewards by increasing the velocity and approaching the boundary of the circle. The safety zone is defined by two parallel plane boundaries that are intersected with the circle. The agent receives a cost equals 1 upon leaving the safety zone. The observation space includes the car's ego states and the sensing of the boundary. We name this task as \texttt{Car-Circle}.

\textbf{Goal Task.} The agent aims to reach the goal buttons while avoiding static surrounding obstacles. After the agent press the correct button, the environment will randomly select a new goal button. The agent will receive positive rewards for moving towards the goal button, and a bonus will be given for successfully reaching the goal. A cost will be penalized for violating safety constraints --- colliding with the static obstacles or pressing the wrong button. The observation space includes the agent's ego states and the sensing information about the obstacles and the goal, which is represented by pseudo LiDAR points. We use a Point robot and a Car robot in this environment. We name them as \texttt{Point-Goal} and \texttt{Car-Goal}.

\textbf{Button Task.} This task is a harder version of \texttt{Goal}, where dynamic obstacles are presented in this task. The dynamic obstacles are moving along a circle continuously, and the agent needs to reach the goal while avoiding both static and dynamic obstacles. We can see that the \texttt{Button} task is harder than \texttt{Circle} and \texttt{Goal}, since the agent is required to infer the surrounding obstacles' states from raw sensing data. We also use a Point robot and a Car robot, and we name the tasks as \texttt{Point-Button} and \texttt{Car-Button}.

\begin{figure*}[!htp]
\centering     
\includegraphics[width=0.9\linewidth]{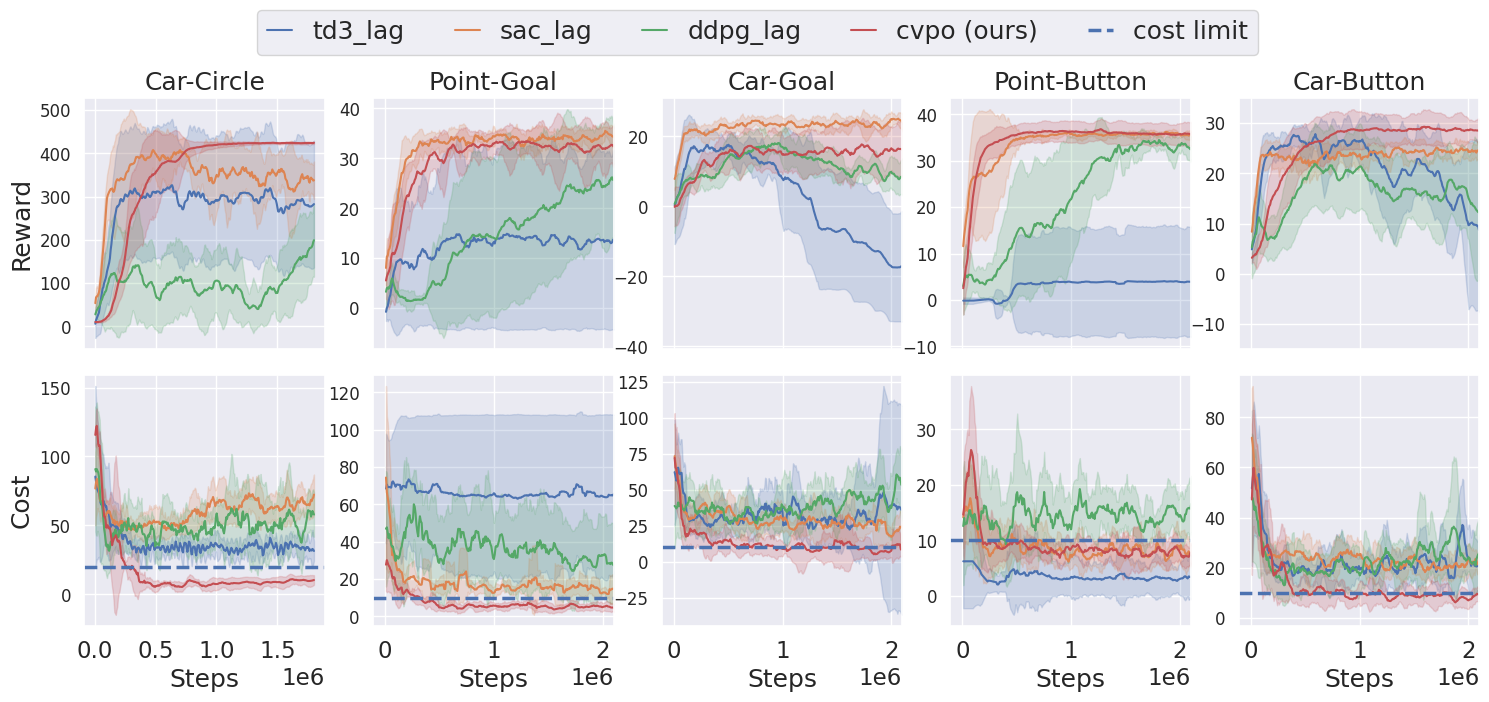}
\caption{Training curves for off-policy baselines comparison. Each column corresponds to an environment. The curves are averaged over 10 random seeds, where the solid lines are the mean and the shadowed areas are the standard deviation.}
\label{fig:off-policy}
\end{figure*}


\subsection{Baselines}
\textbf{Off-policy baselines.} To better understand the role of variational inference, we use three off-policy primal-dual-based baselines. In contrast to the vanilla Lagrangian-based approaches~\cite{ray2019benchmarking}, we use a stronger baseline --- PID-Lagrangian approach~\cite{stooke2020responsive} --- to update the dual variables, which can achieve more stable training.
Since the Lagrangian approach could be easily included in most existing RL methods in principle, we adopt three well-known base algorithms: Soft Actor Critic (SAC)~\cite{haarnoja2018soft}, Deep Deterministic Policy Gradient (DDPG)~\cite{lillicrap2015continuous}, and Twin Delayed DDPG (TD3)~\cite{fujimoto2018addressing}. We name the PID-Lagrangian-augmented safe RL versions as \texttt{SAC-Lag}, \texttt{DDPG-Lag} and \texttt{TD3-Lag}.

\textbf{On-policy baselines.} Constrained Policy Optimization (CPO)~\cite{achiam2017constrained} is used as the on-policy safe RL baseline, since their KL-regularized policy updating algorithm is closely related to us. Additionally, we use two Lagrangian-based baselines that are commonly used in previous works --- \texttt{TRPO-Lag} and \texttt{PPO-Lag}, which are modified from Trust Region Policy Optimization~\cite{schulman2015trust} and Proximal Policy Gradient~\cite{schulman2017proximal}. We also use \texttt{TRPO} as the unconstrained RL baseline to show what is the best task reward performance when ignoring the constraints.

For fair comparison, we use the same network sizes of the policy and critics for all the methods, including CVPO (our method). The safety critics updating rule and the discounting factor are also the same for all off-policy methods. For detailed hyperparameters, please refer to Appendix~\ref{app:hyper-param}.

\begin{figure}[!h]
\vspace{-2mm}
\centering     
\includegraphics[width=0.98\linewidth]{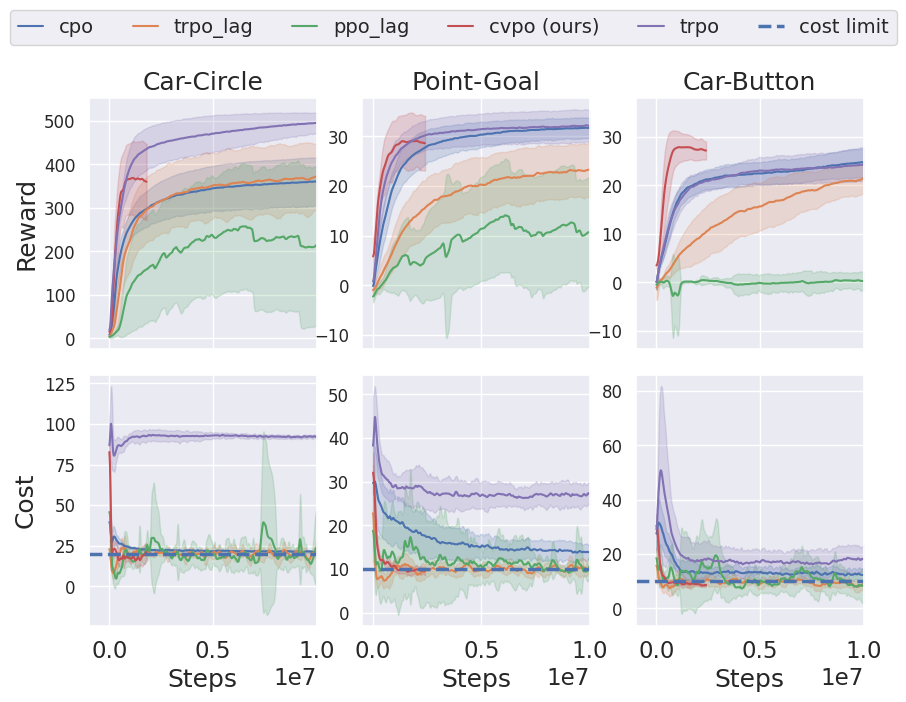}
\caption{Training curves for on-policy baselines comparison.}
\label{fig:on_policy_curves}
\vspace{-4mm}
\end{figure}

\subsection{Results}
We separate the comparison with off-policy and on-policy baselines because off-policy methods are more sample efficient, so the scales for them are different and thus making it hard to distinguish the plots. 
Due to the page limit, we only present some results on a subset of tasks in this section. The complete experimental results for each comparison on all tasks could be found in the Appendix~\ref{app:experiment_result}.

\textbf{Comparison to off-policy baselines.} Fig. \ref{fig:off-policy} shows the training curves for off-policy safe RL approaches. The first row is the undiscounted episodic task reward (the higher, the better). The second row is the undiscounted episodic cost (\# of constraint violations), where the blue dashed line is the target cost threshold. We can clearly see that CVPO outperforms the off-policy baselines in terms of the constraint satisfaction while maintaining high episodic reward, which validates the optimality and feasibility guarantee of our method. Note that all the off-policy baselines use the same network architecture and sizes for $Q_r$ and $Q_c$ critics, and the major difference between them is the policy optimization. The training curves also demonstrate better stability of our approach, as we theoretically analyzed in section \ref{sec:theory_bound}.

\textbf{Comparison to on-policy baselines.} Fig.~\ref{fig:on_policy_curves} shows the training curves for on-policy baselines.
Note that the curves for our method (CVPO) are the same as the ones in Fig.~\ref{fig:off-policy}, and they look to be squeezed because other on-policy baselines require much more samples to converge.
Among the baselines, TRPO-Lag and PPO-Lag have better constraint satisfaction performance than CPO, though PPO-Lag suffers from large variance and low task reward. 
CPO fails to satisfy the constraint for the Goal and Button tasks, which is probably caused by the approximation error during the policy update, as we discussed in section~\ref{sec:estep}. 
Surprisingly, our method can achieve comparable task performance to the \textit{unconstrained RL} baseline -- TRPO (purple curves), while maintaining safety with much fewer constraint violations than on-policy safe RL approaches.

\begin{figure*}[!htp]
\centering     
\includegraphics[width=0.95\linewidth]{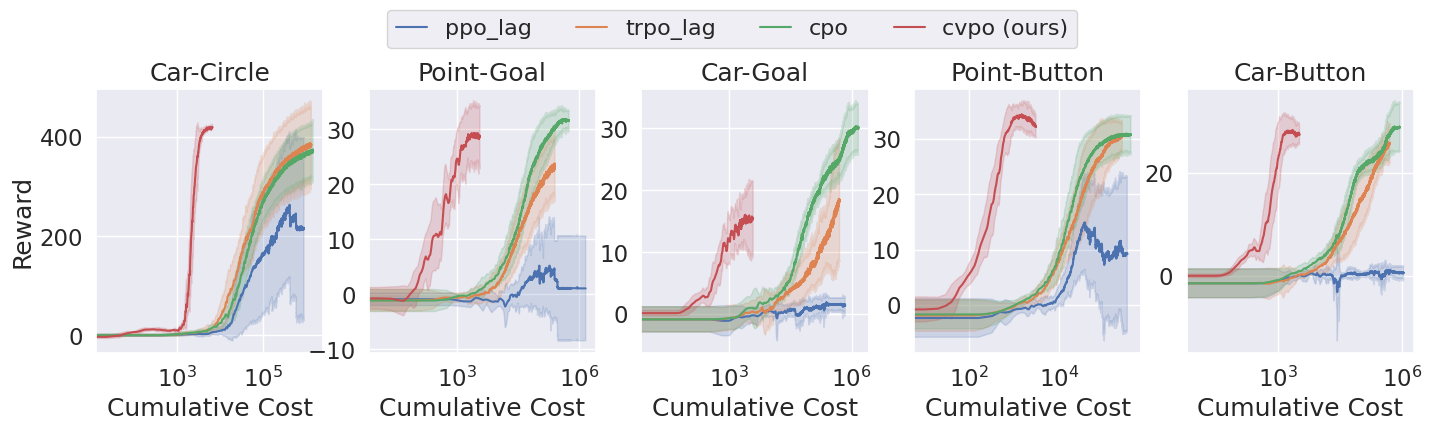}
\caption{Reward versus cumulative cost (log-scale).}
\label{fig:reward_cost}
\vspace{-4mm}
\end{figure*}

Fig.~\ref{fig:reward_cost} demonstrates the \textbf{efficacy} of utilizing each cost -- how much task rewards we could obtain given a budget of constraint violations. The curves that approach to the upper left are better because fewer costs are required to achieve high rewards. We can clearly see that CVPO outperforms baselines with large margin among all tasks -- we use \textbf{100 $\boldsymbol{\sim}$ 1000 times less} cumulative constraint violations to obtain the same task reward. Note that the x-axis is on the log-scale.


\begin{figure}[!h]
\centering     
\includegraphics[width=0.98\linewidth]{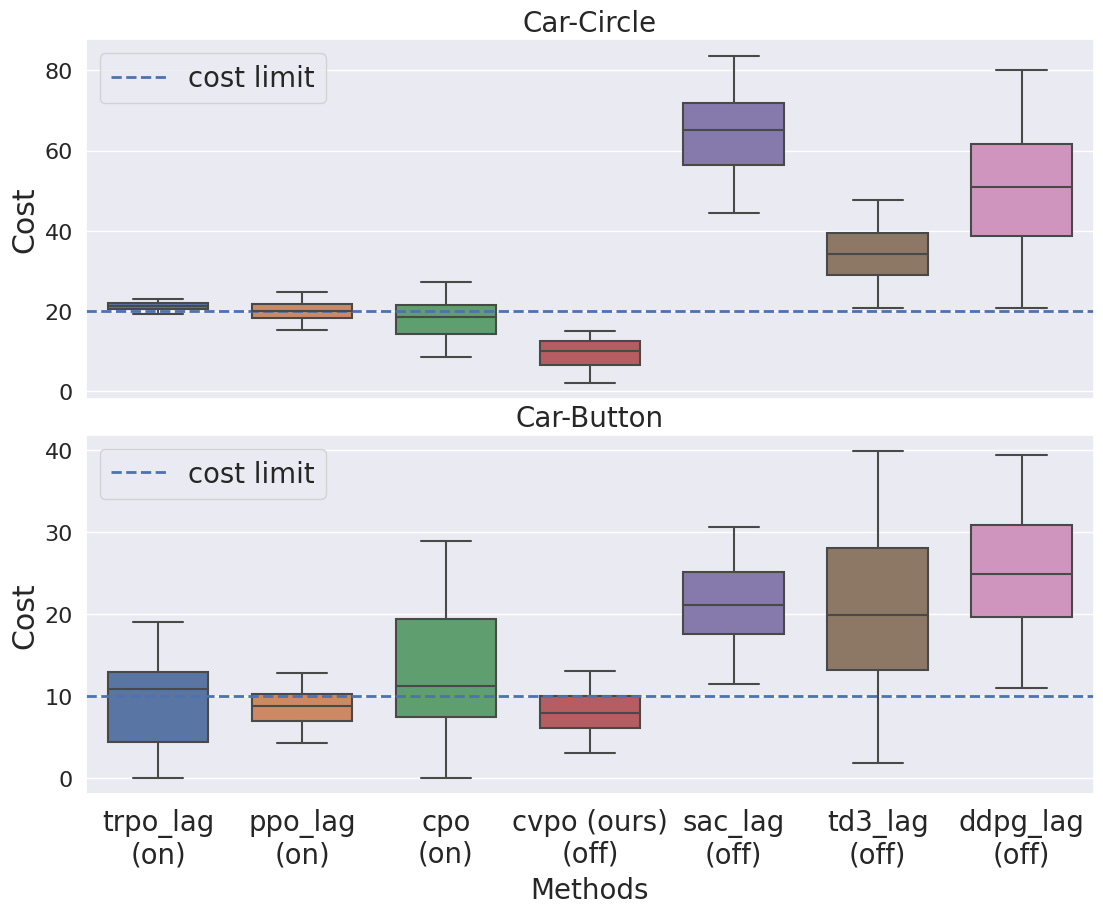}
\caption{Box plot of the convergence cost. (on) and (off) denotes on-policy and off-policy method, respectively.}
\label{fig:box_plot}
\vspace{-4mm}
\end{figure}

\textbf{Convergence cost comparison.} Fig. \ref{fig:box_plot} shows the box plot of each method's constraint satisfaction performance after convergence. We define convergence as the last $20\%$ training steps. The box plot -- also called whisker plot -- presents a five-number summary of the data. The five-number summary is the minimum (the lower solid line), first quartile (the lower edge of the box), median (the solid line in the box), third quartile (the upper edge of the box), and maximum (the upper solid line).
As shown in the figure, on-policy baselines have better constraint satisfaction performance than off-policy baselines, which indicates that extending the primal-dual framework to off-policy settings is non-trivial. We find that our method meets the safety requirement better and with smaller variance than most baselines. Interestingly, we could observe that even \textbf{the third quartile of cost of our approach is below the target cost threshold}, which indicates that CVPO satisfies constraints state-wise, rather than in expectation as baselines. The reason is that we sample a mini-batch from the replay buffer to compute the optimal non-parametric distribution for these states respectively, while baseline approaches directly optimize the policy to satisfy the constraints in expectation. 



\subsection{Ablation study}
\label{sec:ablation_study}
\textbf{The role of KL regularizer in M-step.}
Fig.~\ref{fig:kl_compare} shows the importance of adding the KL constraint during the policy improvement phase. Since the M-step is essentially a supervised learning problem -- fitting the policy with the optimal distribution solved in the E-step, the policy may easily collapse to local optimum action distributions computed from the current batch of data. Without the KL constraints, the worst-case performance bound in Proposition~\ref{Proposition:cost_bound} tends to be infinite and the robustness guarantee in Proposition~\ref{Proposition:slater_gurantee} does not hold, so the agent may fail to satisfy the constraints (Car-Button task). In addition, the policy overfitting prevent the agent from exploring more rewarding trajectories and thus lead to low reward (Car-Circle and Point-Goal tasks).

\begin{figure}[!h]
\centering     
\includegraphics[width=0.95\linewidth]{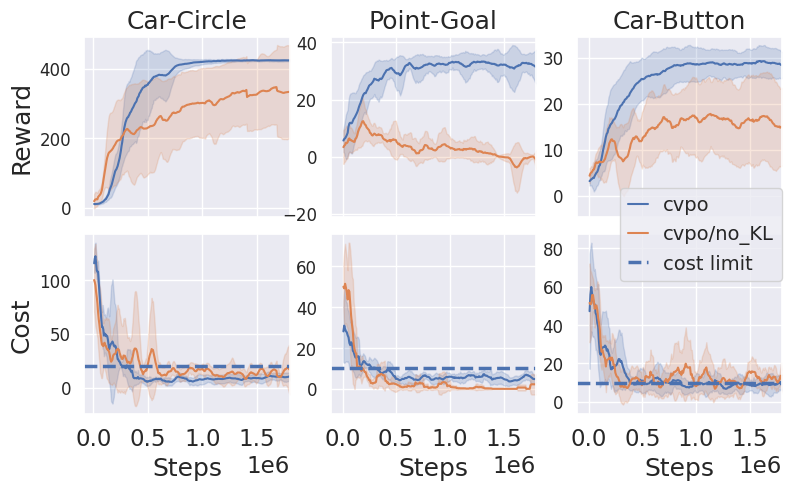}
\caption{Ablation study of the KL constraint in M-step.}
\label{fig:kl_compare}
\vspace{-4mm}
\end{figure}


\section{Conclusion}
We show the safe RL problem can be decomposed into a convex optimization phase with a non-parametric variational distribution and a supervised learning phase. 
We show the unique advantages of constrained variational policy optimization:
1) high \textbf{sample-efficiency} from the off-policy variant of the approach; 2) \textbf{stability} ensured by the monotonic reward improvement guarantee, worst-case constraint violation bound, and the policy updating robustness; and 3) with theoretical \textbf{optimality and feasibility guarantees} by solving a provable and mostly strict convex optimization problem in the E-step. 
We validate the proposed method with extensive experiments, and the results demonstrate the better empirical performance of our method than previous primal-dual approaches in terms of constraint satisfaction and sample efficiency. 

The limitations include that our method would be more computationally expensive since a convex optimization problem need to be solved in the E-step.
In addition, how to improve the critic's estimation of future constraint violations in safe RL is an important and promising topic.
We believe our work can provide a new perspective in the safe RL field and inspire more exciting research in this direction.


\section*{Acknowledgements}
We gratefully acknowledge support from the NSF CAREER CNS-2047454, NSF grant No.1910100, NSF CNS No.2046726, C3 AI, and the Alfred P. Sloan Foundation.
We also would like to thank Jiacheng Zhu, Yufeng Zhang, Gokul Swamy for their help in various aspects of this paper.

\bibliography{citation}
\bibliographystyle{icml2021}

\newpage
\onecolumn
\appendix
\section{Proofs and Discussions}
\subsection{Proof of the evidence lower bound (ELBO) in section~\ref{sec:safe_inference}}
\label{app:elbo}
Denote the probability of a trajectory $\tau$ under the policy $\pi$ as $p_\pi(\tau)=p(s_0)\prod_{t\geq 0}p(s_{t+1} \mid s_t, a_t)\pi(a_t\mid s_t)$, then the lower bound of the log-likelihood of optimality given the policy $\pi$ is
\begin{align}
    \log  p_\pi  (O=1) &= \log \int  p(O=1 \mid \tau)p_\pi(\tau) d\tau \\
    &= \log \E_{\tau \sim q} \left[\frac{p(O=1 \mid \tau)p_\pi(\tau)}{q(\tau) }\right] \\ 
    & \geq \E_{\tau \sim q} \log \frac{p(O=1 \mid \tau)p_\pi(\tau)}{q(\tau) } \label{eq:jensen} \\
    &= \E_{\tau \sim q} \log p(O=1 \mid \tau) + \E_{\tau \sim q} \log \frac{p_\pi(\tau)}{q(\tau) }\\
    &\propto \E_{\tau \sim q} \left[\sum_{t=0}^\infty \gamma^t r_t\right] - \alpha \KL(q(\tau)  \Vert p_\pi(\tau)) = \Jcal(q,\pi)
\end{align}
where inequality (\ref{eq:jensen}) follows Jensen's inequality, $q(\tau)$ is an auxiliary trajectory distribution.

\subsection{Proof and discussion of Theorem \ref{theorem:1} --- the optimal variational distribution and the dual function}
\label{app:dual_function}
Recall that the objective in E-step:
\begin{equation}
\begin{split}
   \max_{q} & \quad \E_{\rho_q}\Big[ \int q(a|s) Q_r^{\pi_{\theta_i}}(s,a) da \Big]\\
    s.t. & \quad \E_{\rho_q}\Big[\int q(a|s) Q_c^{\pi_{\theta_i}}(s,a) da \Big] \leq \epsilon_1 \\
    & \quad \E_{\rho_q}\Big[\KL(q(a|s)\Vert \pi_{\theta_i}(a|s)) \Big] \leq \epsilon_2 \\
    & \quad \int q(a | s) da = 1, \quad \forall s \sim \rho_q.
\end{split}
\end{equation}
Then we prove the optimal variational distribution analytical form and its dual function in Theorem \ref{theorem:1}.
\begin{proof}
To solve the constrained optimization problem, we first convert it to the equivalent Lagrangian function:
\begin{align}
    L(q,\lambda, \eta, \kappa) & = \int \rho_q(s) \int q(a|s) Q_r^{\pi_{\theta_i}}(s,a) da ds \\
    &+ \lambda \left( \epsilon_1 - \int\rho_q(s)\int q(a|s) Q_c^{\pi_{\theta_i}}(s,a) dads \right) \\ 
    & + \eta \left(\epsilon_2 - \int\rho_q(s)\int q(a|s) \log \frac{q(a|s)}{\pi_{\theta_i}(a|s)}dads \right) \\
    &+ \kappa \left(1 - \int\rho_q(s)\int q(a|s)dads \right),
    \label{eq:lagrangian}
\end{align}
where $\lambda, \eta, \kappa$ are the Lagrange multipliers for the constraints. 
Since the objective is linear and all constraints are convex (note that KL is convex), the E-step optimization problem is convex.
Then we obtain the equivalent dual problem:
\begin{equation}
    \min_{\lambda, \eta, \kappa} \max_q L(q, \lambda, \eta, \kappa).
    \label{eq:dual_formulation}
\end{equation}

Take the derivative of Lagrangian function w.r.t $q$:
\begin{equation}
    \frac{\partial L}{\partial q} = Q_r^{\pi_{\theta_i}}(s,a) - \lambda Q_c^{\pi_{\theta_i}}(s,a) - \eta - \kappa - \eta\log\frac{q(a|s)}{\pi_{\theta_i}(a|s)}.
    \label{eq:q_partial}
\end{equation}

Let Eq. (\ref{eq:q_partial}) be zero, we have the form of the optimal $q$ distribution:
\begin{align}
    q^*(a|s) = \pi_{\theta_i}(a|s) \exp\left(\frac{Q_r^{\pi_{\theta_i}}(s,a) - \lambda Q_c^{\pi_{\theta_i}}(s,a)}{\eta}\right) \exp\left(-\frac{\eta+\kappa}{\eta}\right),
    \label{eq:optimalq}
\end{align}
where $\exp\left(-\frac{\eta+\kappa}{\eta}\right)$ could be viewed as a normalizer for $q(a|s)$ since it is a constant that is independent of $q$. Thus, we obtain the following form of the normalizer by integrating the optimal $q$:
\begin{align}
    \exp\left(\frac{\eta+\kappa}{\eta}\right) = \int \pi_{\theta_i}(a|s) \exp\left(\frac{Q_r^{\pi_{\theta_i}}(s,a) - \lambda Q_c^{\pi_{\theta_i}}(s,a)}{\eta}\right) da,
    \label{eq:normalizer}
\end{align}
\begin{align}
    \frac{\eta+\kappa}{\eta} = \log\int \pi_{\theta_i}(a|s) \exp\left(\frac{Q_r^{\pi_{\theta_i}}(s,a) - \lambda Q_c^{\pi_{\theta_i}}(s,a)}{\eta}\right) da.
    \label{eq:normalizer2}
\end{align}
Take the optimal $q$ distribution in Equation~(\ref{eq:optimalq}) and $\frac{\eta+\kappa}{\eta}$ in Equation~(\ref{eq:normalizer2}) back to the Lagrangian function~(\ref{eq:lagrangian}), we can find that most of the terms are cancelled out, and obtain the dual function $g(\eta, \lambda)$,
\begin{align}
    g(\eta, \lambda) = \lambda \epsilon_1 + \eta \epsilon_2 + \eta \int \rho_q(s) \log \int \pi_{\theta_i}(a|s) \exp\left(\frac{Q_r^{\pi_{\theta_i}}(s,a) - \lambda Q_c^{\pi_{\theta_i}}(s,a)}{\eta}\right) da ds.
    \label{eq:dual}
\end{align}
The optimal dual variables are calculated by
\begin{align}
   \eta^*, \lambda^* = \arg\min_{\eta, \lambda} g(\eta, \lambda).
\label{eq:optimal_dual}
\end{align}
\end{proof}
A good property is that the dual function is convex (as we will prove in Appendix~\ref{app:convex}), so we could use off-the-shelf convex optimization tools to solve the dual problem. Also, under the Slater's assumption (\ref{assumption:slater}) --- there exists a feasible distribution $\Bar{q} \in \Pi_\Qcal$ within the trust region of the old policy $\pi_{\theta_i}$: $\KL(\Bar{q}\Vert \pi_{\theta_i})\leq \epsilon_2$, we could solve the optimal dual variables efficiently.

\begin{Lemma}
(Strong duality). If the Slater condition in Assumption (\ref{assumption:slater}) holds, then the strong duality holds for the original problem (\ref{eq:estep_objective}) and the dual problem (\ref{eq:dual_formulation}).
\label{lemma:strong_duality}
\end{Lemma}
Lemma \ref{lemma:strong_duality} implies that the optimality of our closed form non-parametric distribution in Eq. (\ref{eq:optimalq}) could be guaranteed, since there is zero duality gap between the primal and dual problem.

Based on above proofs and lemma, we could easily obtain Theorem \ref{theorem:1}.
Note that Theorem \ref{theorem:1} is similar to the one in FOCOPS~\cite{zhang2020first} but with several major differences: 1) they use a parametric policy as the optimization objective in (\ref{eq:optimalq_theorem}) instead of our non-parametric variational distribution. The parametrized form makes the problem hard to solve analytically, and thus first-order approximation over the policy parameters is required; 2) we provide an analytical form of the dual function and prove it is convex, such that the dual variables $\eta, \lambda$ are guaranteed to be optimal by convex optimization, while they use a gradient-based method to solve the dual variables, which is similar to the Lagrangian-based methods in the literature~\cite{stooke2020responsive}; 3) they consider the on-policy advantage estimations in their objective (\ref{eq:optimalq_theorem}), while we allow off-policy evaluation to update the policy, which should be more sample-efficient.

\subsection{Proof and discussion of Theorem \ref{theorem:dual_convex} --- the dual function's convexity and the condition of strong convexity}
\label{app:convex}
Note the dual function is the supremum of Lagrangian function w.r.t $q$:
\begin{equation}
    g(\eta, \lambda) = \max_q L(q,\lambda, \eta, \kappa),
\end{equation}
where $\kappa$ can be calculated by Eq. (\ref{eq:normalizer}) given $q$. While the convexity of dual function $g$ has been proven in previous literature (Proposition 1, section 8.3 in \citep{luenberger1997optimization}), we further derive the conditions of strong convexity.

Recall the dual function (we omit $\theta$ to simplify the notation):
\begin{align}
    g(\eta, \lambda) = \lambda \epsilon_1 + \eta \epsilon_2 + \eta \int \rho_q(s) \log \int \pi(a|s) \exp\left(\frac{Q_r(s,a) - \lambda Q_c(s,a)}{\eta}\right) da ds .
\end{align}
Observe that the outer integral is independent of $\eta, \lambda$, so the convexity of $g(\eta, \lambda)$ is the same as:
\begin{align}
    \Bar{g}(\eta, \lambda) = \lambda \epsilon_1 + \eta \epsilon_2 + \eta \log \int \pi(a|s) \exp\left(\frac{Q_r(s,a) - \lambda Q_c(s,a)}{\eta}\right) da.
    \label{eq:dual_per_state}
\end{align}
  
\begin{Remark}
Interestingly, both dual functions could be used in off-policy implementation by tuning the batch size hyper-parameter.
We could view $g(\eta, \lambda)$ as solving the dual variables through a batch of data, while $\Bar{g}(\eta, \lambda)$ is a per-state variant --- batch size equals one. Namely, $\Bar{g}(\eta, \lambda)$ aims to compute the optimal dual variables for each state. Practically, we prefer using the batch version because computing the dual for each state is computational expensive and sensitive to the Q-value estimation error.
\end{Remark}

To further prove the strong convexity conditions of $\Bar{g}$, we first introduce the following lemmas.

\begin{Lemma} (Lagrange's Identity)
Given two functions $f, g$ such that $\int_a^b f^2(x)dx<+\infty, \int_a^b g^2(x)dx<+\infty$, the following equation holds,
\begin{equation}
    \int_a^b f^2(x)dx\int_a^b g^2(x)dx-\left(\int_a^b f(x)g(x)dx\right)^2 = \frac{1}{2}\int_a^b \int_a^b \left(f(x)g(y)-g(x)f(y)\right)^2 dxdy.
\end{equation}
\label{lemma:Lagrange_identity}
\end{Lemma}

\begin{proof}
\begin{align}
    \text{LHS} &= \int_a^b\int_a^b f^2(x) g^2(y)dxdy - \int_a^b f(x)g(x)dx\int_a^b f(y)g(y)dy\\
    &= \frac{1}{2}\int_a^b\int_a^b f^2(x) g^2(y)dxdy + \frac{1}{2}\int_a^b\int_a^b f^2(y) g^2(x)dxdy - \int_a^b\int_a^b f(x)g(x) f(y)g(y) dxdy \\
    &= \frac{1}{2}\int_a^b \int_a^b \left(f(x)g(y)-g(x)f(y)\right)^2 dxdy =\text{RHS}.
\end{align}
\end{proof}
According to the convexity of dual function, the Hessian matrix of $\Bar{g}(\lambda, \eta)$ is always positive semi-definite. Extending to this generic property, the following lemma gives the conditions of strict positive definiteness of Hessian matrix in constrained RL problem.
\begin{Lemma}
The Hessian matrix $\rmH$ of $\Bar{g}(\lambda, \eta)$ is strictly positive definite when (1) $Q_r^{\theta_i}(s,\cdot), Q_c^{\theta_i}(s,\cdot)$ are not constant functions; (2) $\forall C\in\mathbb{R}, \exists a_0, s.t.  Q_r^{\theta_i}(s,a_0)\neq C\cdot Q_c^{\theta_i}(s,a_0)$; and (3) $\lambda^* < +\infty$, where $\lambda^*$ is the optimal dual variables in Eq. (\ref{eq:optimal_dual}).
\label{proposition:Hessian}
\end{Lemma} 

\begin{proof}
For simplicity, let $M(a)\doteq \pi(a|s) \exp\left(\frac{Q_r(s,a) - \lambda Q_c(s,a)}{\eta}\right)$, then $\Bar{g}(\eta, \lambda) = \lambda \epsilon_1 + \eta \epsilon_2 + \eta \log \int  M(a) da $,
\begin{equation}
\begin{cases}
\cfrac{\partial M(a)}{\partial \lambda} &= M(a)\left(-\cfrac{Q_c(s,a)}{\eta}\right)\\
\cfrac{\partial M(a)}{\partial \eta} &= M(a)\left(-\cfrac{Q_r(s,a)-\lambda Q_c(s,a)}{\eta^2}\right)
\end{cases}
\end{equation}

\begin{equation}
\Rightarrow 
\begin{cases}
\cfrac{\partial \Bar{g}(\lambda, \eta)}{\partial \lambda} &= \epsilon_1 - \cfrac{\int M(a) Q_c(s,a) da}{\int M(a) da}\\
\cfrac{\partial \Bar{g}(\lambda, \eta)}{\partial \eta} &= \epsilon_2 + \log{\int M(a) da} - \cfrac{\int M(a)[Q_r(s,a) - \lambda Q_c(s,a)] da}{\eta \int M(a) da}
\end{cases}
\label{eq:1st-derivate}
\end{equation}

\begin{equation}
\Rightarrow 
\begin{cases}
\rmH_{1,1} = \cfrac{\partial^2 \Bar{g}(\lambda, \eta)}{\partial \lambda^2} &= \cfrac{\int M(a) (Q_c(s,a))^2 da\int M(a)da - \left(\int M(a) Q_c(s,a)da\right)^2}{\eta (\int M(a) da)^2} \\
\rmH_{2,2} = \cfrac{\partial^2 \Bar{g}(\lambda, \eta)}{\partial \eta^2} &= \cfrac{\int M(a) [Q_r(s,a) - \lambda Q_c(s,a)]^2 da\int M(a)da - \left(\int M(a) [Q_r(s,a) - \lambda Q_c(s,a)] da\right)^2}{\eta^3 (\int M(a) da)^2} \\
\rmH_{1,2} = \cfrac{\partial^2 \Bar{g}(\lambda, \eta)}{\partial \lambda \partial \eta} &= \rmH_{2,1} = \cfrac{\partial^2 \Bar{g}(\lambda, \eta)}{\partial \eta \partial \lambda} = \cfrac{\int M(a) da \int M(a) [Q_r(s,a)-\lambda Q_c(s,a)]Q_c(s,a) da}{\eta^2 (\int M(a) da)^2} \\
&\quad- \cfrac{\int M(a) Q_c(s,a)da \int M(a) [Q_r(s,a)-\lambda Q_c(s,a)]da}{\eta^2 (\int M(a) da)^2}.
\end{cases}
\end{equation}
Therefore, $\rmH$ is positive semi-definite with following two conditions,
\begin{equation}
\rmH_{1,1}\geq 0,\quad \rmH_{1,1}\rmH_{2,2}-\rmH_{1,2}^2\geq 0.
\label{eq:2_conditions}
\end{equation}
Note $\rmH$ is strictly positive definite when the equality does not hold. We will first prove the inequality holds and further discuss the condition of strict positive definiteness later.

By Cauchy–Schwarz inequality, 
\begin{equation}
\int M(a) (Q_c(s,a))^2 da\int M(a)da \geq \left(\int M(a) Q_c(s,a)da\right)^2 \Rightarrow \rmH_{1,1} \geq 0.
\end{equation}

i.e., the first condition in Eq. (\ref{eq:2_conditions}) holds. To prove the second condition, we use shorthand $M=M(a), Q_r=Q_r(s,a),Q_c=Q_c(s,a)$ for simplicity,

\begin{equation}
    \begin{aligned}
    & \rmH_{1,1}\rmH_{2,2}-\rmH_{1,2}^2\geq 0\\
    \Leftrightarrow & \left[ \int M Q_c^2 da\int M da - \left(\int M Q_c da\right)^2 \right] \left[ \int M (Q_r - \lambda Q_c)^2 da\int M da - \left(\int M (Q_r - \lambda Q_c) da\right)^2 \right] \geq \\
    & \quad \left[ \int M Q_c da \int M(Q_r-\lambda Q_c)da - \int M da \int M(Q_r- \lambda Q_c) Q_c da \right]^2 .
    \label{eq:hessian_pivots}
\end{aligned}
\end{equation}

By lemma \ref{lemma:Lagrange_identity}, we have
\begin{align}
 &\int M Q_c^2 da\int M da - \left(\int M Q_c da\right)^2 \\
= &\frac{1}{2} \iint \left(\sqrt{M(a_1)}Q_c(s,a_1)\sqrt{M(a_2)}- \sqrt{M(a_2)}Q_c(s,a_2)\sqrt{M(a_1)} \right)^2 da_1 da_2\\
= &\frac{1}{2} \iint M(a_1)M(a_2)\left(Q_c(s,a_1)- Q_c(s,a_2)\right)^2 da_1 da_2, 
\end{align}
and
\begin{align}
&\int M (Q_r - \lambda Q_c)^2 da\int M da - \left(\int M (Q_r - \lambda Q_c) da\right)^2 \\
= &\frac{1}{2} \iint \left(\sqrt{M(a_1)}(Q_r(s,a_1)-\lambda Q_c(s,a_1))\sqrt{M(a_2)}- \sqrt{M(a_2)}(Q_r(s,a_2)-\lambda Q_c(s,a_2))\sqrt{M(a_1)} \right)^2 da_1 da_2\\
= &\frac{1}{2} \iint M(a_1)M(a_2)\Big[(Q_r(s,a_1)-\lambda Q_c(s,a_1))- (Q_r(s,a_2)-\lambda Q_c(s,a_2))\Big]^2 da_1 da_2. 
\end{align}
Therefore, apply Cauchy–Schwarz inequality to the LHS of Eq. (\ref{eq:hessian_pivots}), we have

\begin{align}
\text{LHS} &= \frac{1}{2} \iint M(a_1)M(a_2)\left(Q_c(s,a_1)- Q_c(s,a_2)\right)^2 da_1 da_2 \\ 
& \quad\quad \cdot \frac{1}{2} \iint M(a_1)M(a_2)\Big[(Q_r(s,a_1)-\lambda Q_c(s,a_1))- (Q_r(s,a_2)-\lambda Q_c(s,a_2))\Big]^2 da_1 da_2 \\
&\geq \frac{1}{4} \left[\iint M(a_1)M(a_2)(Q_c(s,a_1)- Q_c(s,a_2)) \right. \\
& \quad\quad \left. \cdot \left((Q_r(s,a_1)-\lambda Q_c(s,a_1))- (Q_r(s,a_2)-\lambda Q_c(s,a_2))\right) da_1 da_2\right]^2, \text{(by Cauchy–Schwarz inequality)}\\
& = \frac{1}{4} \Big[\iint M(a_1)M(a_2)\big[Q_c(s,a_1)(Q_r(s,a_1)-\lambda Q_c(s,a_1))- Q_c(s,a_2)(Q_r(s,a_1)-\lambda Q_c(s,a_1))  \\
& \quad\quad - Q_c(s,a_1)(Q_r(s,a_2)-\lambda Q_c(s,a_2)) + Q_c(s,a_2)(Q_r(s,a_2)-\lambda Q_c(s,a_2))\big] da_1 da_2\Big]^2\\
& = \frac{1}{4} \left[2\iint M(a_1)M(a_2)\left(Q_c(s,a_1)(Q_r(s,a_1)-\lambda Q_c(s,a_1))- Q_c(s,a_1)(Q_r(s,a_2)-\lambda Q_c(s,a_2)) \right) da_1 da_2\right]^2\\
& = \left[\int M(a_1)Q_c(s,a_1)(Q_r(s,a_1)-\lambda Q_c(s,a_1))da_1 \int M(a_2)da_2 \right. \\
& \quad\quad \left.- \int M(a_1) Q_c(s,a_1) da_1 \int M(a_2)(Q_r(s,a_2)-\lambda Q_c(s,a_2)) da_2\right]^2 = \text{RHS}.
\end{align}
Note that we use the Cauchy-Schwarz inequality in a middle step. 
We could easily check that when $Q_c$ is a constant function or $\lambda \xrightarrow[]{} +\infty$, then the equality holds.
Otherwise, the equality holds if and only if:
\begin{equation}
    \frac{(Q_r(s,a_1)-\lambda Q_c(s,a_1))- (Q_r(s,a_2)-\lambda Q_c(s,a_2))}{Q_c(s,a_1) - Q_c(s,a_2)} = \frac{Q_r(s,a_1)- Q_r(s,a_2)}{Q_c(s,a_1) - Q_c(s,a_2)} - \lambda = C,
\label{eq:strict_convex_condition}
\end{equation}

where $C$ is a constant. It could be achieved by setting $Q_r$ be a constant function or let $Q_r$ be proportional to $Q_c$. 
We summarize the above conditions as follows: (1) $Q_r(s,\cdot)$ or $Q_c(s,\cdot)$ is a constant function; or (2) $\exists C\in\mathbb{R}, s.t. \forall a, Q_r(s,a) = C  Q_c(s,a)$; or (3) $\lambda \xrightarrow[]{} + \infty$. Each condition has corresponding meaning, as we explain in the following remarks.
\end{proof}
\begin{Remark}
For the first condition, if $Q_r(s,\cdot)$ is a constant function, then the objective in (\ref{eq:estep_objective}) will always be a constant --- the optimization becomes meaningless since all the distributions that within the trust region should be the same; if $Q_c(s,\cdot)$ is a constant function, then the safety constraint will be inactive, since no policy could change the feasibility status. Note that the Hessian element $\rmH_{1,1}$ will be 0 if $Q_c(s,\cdot)$ is a constant, which means that the gradient of the safety constraint dual variable $\lambda$ will either be a positive constant when the problem is feasible (tends to make $\lambda\xrightarrow[]{} 0$) or a negative constant when the problem is infeasible (tends to make $\lambda\xrightarrow[]{} +\infty$). 
\end{Remark}
\begin{Remark}
The second condition indicates that if the task reward $Q_r$ value is proportional to the cost $Q_c$ value everywhere, then multiple optimal solutions may exist on the KL constraint boundary or the safety constraint boundary. Intuitively, the problem becomes a bounded convex optimization --- maximizing the risks $\E_q[Q_c(s,a)]$ within the trust region defined by $\epsilon_2$ and a ball defined by $\epsilon_1$. Though the solutions may not be unique, the optimality could still be guaranteed as long as the Slater's condition \ref{assumption:slater} holds.
\end{Remark}
\begin{Remark}
The third condition is equivalent to the violation of our Slater's condition assumption \ref{assumption:slater} --- no distribution exists within the trust region that satisfies the safety constraints. In that sense, the gradient of the dual variable $\lambda$ will always be negative, and thus $\lambda \xrightarrow[]{} +\infty$, which tends to impose huge penalties on safety critics $Q_c$. In practice, this rarely happens since we could select a large KL threshold for E-step and choose a much smaller trust region size for M-step, such that $n$-step robustness could be guaranteed and the Slater's condition could be easily met, as we show in Appendix \ref{app:n_step_robustness}.
\end{Remark}

In summary, $\Bar{g}(\lambda, \eta)$ is convex by Lemma \ref{proposition:Hessian} and strictly convex when the equality does not hold in Eq. (\ref{eq:strict_convex_condition}). 
The strict convexity analysis ensures the uniqueness of the optimal solution in most situations and provides some other theoretical implications of the E-M procedure, such as the convergence to a stationary point \cite{sriperumbudur2009integral}.
In addition, as long as the Slater's condition in Assumption \ref{assumption:slater} holds, we could obtain the optimal solution of the dual problem efficiently via any convex optimization techniques, though multiple optimal solutions may exist if Eq. (\ref{eq:strict_convex_condition}) holds.

\subsection{Proof of the M-step --- KL regularized policy improvement}
\label{app:mstep}
As shown in section~\ref{sec:safe_inference} and \cite{abdolmaleki2018maximum}, given the optimal variational distribution $q_i^*$ from the E-step, the M-step objective is
\begin{equation}
\begin{split}
\theta_{i+1} = \arg\max_\theta \E_{\rho_q}\Big[\alpha \E_{q^*_i(\cdot|s)} \big[ \log \pi_\theta (a|s)\big]\Big] + \log p(\theta)
\end{split}
\label{eq:app_mstep}
\end{equation}
which is a Maximum A-Posteriori (MAP) problem.
Consider a Gaussian prior around the old policy parameter $\theta_i$, we have 
$$\theta \sim \Ncal(\theta_i, \frac{ F_{\theta_i}}{\alpha\beta})$$
where $F_{\theta_i}$ is the Fisher information matrix and $\beta$ is a positive constant. 
With the Gaussian prior, the objective (\ref{eq:app_mstep}) becomes
\begin{equation}
\begin{split}
\theta_{i+1} &= \arg\max_\theta \alpha  \E_{\rho_q}\Big[\E_{q^*_i(\cdot|s)} \big[ \log \pi_\theta (a|s)\big]\Big] - \alpha \beta(\theta-\theta_i)^TF_{\theta_i}^{-1}(\theta-\theta_i) \\
& = \arg\max_\theta  \E_{\rho_q}\Big[\E_{q^*_i(\cdot|s)} \big[ \log \pi_\theta (a|s)\big]\Big] - \beta(\theta-\theta_i)^TF_{\theta_i}^{-1}(\theta-\theta_i)
\end{split}
\end{equation}
where we could observe that $(\theta-\theta_i)^TF_{\theta_i}^{-1}(\theta-\theta_i)$ is the second order Taylor expansion of $\E_{\rho_q}\big[\KL(\pi_{\theta_i}(a|s)\Vert \pi_{\theta}(a|s)) \big]$. Thus, we could generalize the above objective to the KL-regularized one:
\begin{equation}
\begin{split}
    \max_\theta \E_{\rho_q}\Big[ \E_{q^*_i(\cdot|s)} \big[ \log \pi_\theta (a|s)\big] - \beta \KL(\pi_{\theta_i}\Vert \pi_{\theta}) \Big].
\end{split}
\end{equation}
Similar to the E-step, we could convert the soft KL regularizer to a hard KL constraint:
\begin{equation}
\begin{split}
    &\max_\theta \quad \E_{\rho_q}\Big[ \E_{q^*_i(\cdot|s)} \big[ \log \pi_\theta (a|s)\big]\Big] \\
    & s.t. \quad \E_{\rho_q}\big[ \KL(\pi_{\theta_i}(a|s)\Vert \pi_{\theta}(a|s)) \big] \leq \epsilon.
\end{split}
\end{equation}

\subsection{Proof of Proposition \ref{Proposition:improvement} --- the ELBO improvement guarantee}
\label{app:ELBO_improve}
Recall that the optimality of policy $\pi$ is lower bounded by the following ELBO objective
$$
\Jcal(q,\theta) = \E_{\tau\sim q} \left[ \sum_{t=0}^\infty \left(\gamma^t r_t - \alpha\KL(q(\cdot| s_t)  \| \pi_\theta(\cdot|s_t))\right)\right] + \log p(\theta).
$$
We improve the policy by optimizing the ELBO alternatively via EM. Thus, we will prove the monotonic improvement guarantee of ELBO at the $i$-th training iteration with the following assumption.
\begin{Assumption}
\label{assumption:two_e_step}
The Slater's condition holds for both $\pi_{\theta_{i-1}}$ and $\pi_{\theta_i}$.
\end{Assumption}
The above assumption indicates a well-optimized policy in the M-step. In addition, it ensures the variational distributions $q_{i-1}$ and $q_{i}$ are feasible. Note that an infeasible variational distribution $q_{i-1}$ may lead to arbitrarily high reward return.
With this assumption, we prove the ELBO improvement for E-step and M-step separately.

\begin{proof}
\textbf{E-step}: By the definition of E-step, we improve the ELBO w.r.t $q$. Since $q_{i-1}$ is feasible (reward return is bounded) and $\pi_{\theta_i}$ satisfies the Slater's condition, we can prove the E-step update will increase ELBO by Theorem~\ref{theorem:1}: 
$$
\begin{aligned}
&\begin{cases}
\begin{aligned}
q_{i} &= \argmax_{q\in \Pi_{\Qcal}^{\epsilon_1}} \E_{\rho_q}\left[\E_{a\sim q(\cdot|s)} \left[ Q_r^{q_i}(s,a)\right] - \alpha \KL[q(\cdot|s) \| \pi_{\theta_i}(\cdot|s)] \right]  \quad \text{(By Slater's condition for $\pi_{\theta_i}$)} \\
&= \argmax_{q\in \Pi_{\Qcal}^{\epsilon_1}} \E_{\tau\sim q} \left[ \sum_{t=0}^\infty \left(\gamma^t r_t - \alpha\KL(q(\cdot| s_t)  \| \pi_\theta(\cdot|s_t))\right)\right]\\
&= \argmax_{q\in \Pi_{\Qcal}^{\epsilon_1}} \Jcal(q,\theta_i)
\end{aligned}\\
q_{i-1} \in \Pi_{\Qcal}^{\epsilon_1} \quad \text{(By Slater's condition for $\pi_{\theta_{i-1}}$)}
\end{cases}\\
\Rightarrow & \Jcal(q_i,\theta_i) \geq \Jcal(q_{i-1},\theta_i).
\end{aligned}
$$
Therefore, as long as assumption \ref{assumption:two_e_step} holds, the ELBO will increase monotonically in terms of $q$.

\textbf{M-step}: By definition in Eq.~(\ref{eq:mstep}), we update $\theta$ by
$$
\begin{aligned}
\theta_{i+1} &= \argmax_{\theta} \E_{\rho_{q_i}}\left[\alpha \E_{a\sim q_i(\cdot|s)} \left[ \log \pi_\theta (a|s)\right]\right] + \log p(\theta) \\
&= \argmax_{\theta} \E_{\rho_{q_i}}\left[- \alpha\KL[q_i(\cdot| s_t)  \Vert \pi_\theta(\cdot|s)]\right] + \log p(\theta) \\
&= \argmax_{\theta} \E_{\rho_{q_i}} \left[ \E_{a\sim q_i(\cdot|s)} \big[ Q_r^{q_i}(s,a)\big] - \alpha\KL[q_i(\cdot| s_t)  \Vert \pi_\theta(\cdot|s)]\right] + \log p(\theta)\\
&= \argmax_{\theta} \Jcal(q_i, \pi_{\theta}).
\end{aligned}
$$
Therefore, we have:
$
\Jcal(q_i, \pi_{\theta_{i+1}}) \geq \Jcal(q_{i}, \pi_{\theta_{i}}) 
$. 
Combining all the above together, we have
$$
\Jcal(q_i, \pi_{\theta_{i+1}}) \geq \Jcal(q_{i}, \pi_{\theta_{i}}) \geq \Jcal(q_{i-1}, \pi_{\theta_{i}}) .
$$
\end{proof}

\begin{Remark}
The Slater's condition assumption is critical for monotonic policy improvement guarantee, since as we have shown in Appendix \ref{app:dual_function} Remark 4, violation of Slater's condition may lead to large penalty for constraint violations, and then in E-step the variational distribution is updated to reduce the cost or reconcile to the feasible set instead of improving the reward return. However, the $n$-step robustness (Appendix \ref{app:n_step_robustness}) guarantees the updated policies in $n$ consecutive iterations satisfy the Slater's condition if we choose proper KL constraint thresholds and the initial policy is feasible, which indicates a monotonic ELBO improvement guarantee during these $n$ policy updating iterations.
\end{Remark}

\subsection{Proof of Proposition \ref{Proposition:cost_bound} --- worst-case constraint satisfaction bound}
\label{app:worst_case}
The Corollary 2 and 3 in CPO~\cite{achiam2017constrained} connect the difference in cost returns between two policies to the divergence between them,
\begin{equation}
\begin{aligned}
J_c(\pi') - J_c(\pi) \leq \frac{1}{1-\gamma} \E_{s\sim \rho_{\pi}, a\sim \pi'} [A_c^{\pi}(s,a)] + \frac{ 2 \gamma \delta_{c}^{\pi'} }{(1-\gamma)^2}\sqrt{\frac{1}{2}\E_{s\sim \rho_{\pi}}\left[\KL[\pi'(\cdot|s) \| \pi(\cdot|s)]\right]} 
\end{aligned}
\end{equation}
where $\delta_{c}^{\pi'} = \max_s \left| \E_{a\sim \pi'}[A_c^{\theta_i}(s,a)] \right|$, and  $A_c^{\theta_i}(s,a)$ denotes the advantage function of cost.

Note that the M-step does not involve cost constraint when updating $\theta$. Therefore, we obtain the following worse-case bound:
\begin{equation}
\begin{aligned}
J_c(\pi_{\theta_{i+1}}) &\leq J_c(\pi_{\theta_{i}}) + \frac{1}{1-\gamma} \E_{s\sim \rho_{\pi_{\theta_{i}}}, a\sim \pi_{\theta_{i+1}}} [A_c^{\pi_{\theta_{i}}}(s,a)] + \frac{ 2 \gamma \delta_{c}^{\pi_{\theta_{i+1}}} }{(1-\gamma)^2}\sqrt{\frac{1}{2}\E_{s\sim \rho_{\pi_{\theta_{i}}}}\left[\KL[\pi_{\theta_{i+1}}(\cdot|s) \| \pi_{\theta_{i}}(\cdot|s)]\right]} \\
& \leq J_c(\pi_{\theta_{i}}) + \frac{1}{1-\gamma} \delta_{c}^{\pi_{\theta_{i+1}}} + \frac{ 2 \gamma \delta_{c}^{\pi_{\theta_{i+1}}} }{(1-\gamma)^2}\sqrt{\frac{1}{2}\epsilon} \\
& = J_c(\pi_{\theta_{i}}) + \frac{[(1-\gamma)+\sqrt{2\epsilon}\gamma]\delta_c^{\pi_{\theta_{i+1}}}}{(1-\gamma)^2}.
\end{aligned}
\end{equation}

When the old policy $\pi_{\theta_{i}}\in \Pi_{\Qcal}^{\epsilon_1}$, we further have
\begin{equation}
J_c(\pi_{\theta_{i+1}}) \leq \epsilon_1 + \frac{[(1-\gamma)+\sqrt{2\epsilon}\gamma]\delta_c^{\pi_{\theta_{i+1}}}}{(1-\gamma)^2}.
\end{equation}

\subsection{Proof and discussion of Proposition \ref{Proposition:slater_gurantee} --- policy updating robustness}
\label{app:n_step_robustness}
Recall the Proposition \ref{Proposition:slater_gurantee} --- 
Suppose $\pi_{\theta_{i}} \in \Pi_\Qcal^{\epsilon_1}$. $\pi_{\theta_{i+1}}$ and $\pi_{\theta_{i}}$ are related by the M-step. If $\epsilon < \epsilon_2$, where $\epsilon, \epsilon_2$ are the KL threshold in M-step and E-step respectively, then the variational distribution $q_{i+1}^*$ in the next iteration is guaranteed to be feasible and optimal.

\begin{proof}
Since $\epsilon < \epsilon_2$, the KL divergence between $\pi_{\theta_{i+1}}$ and $\pi_{\theta_i}$ $\KL(\pi_{\theta_{i+1}} \Vert \pi_{\theta_i}) \leq \epsilon < \epsilon_2$. Thus, the Slater condition \ref{assumption:slater} holds for $\pi_{\theta_{i+1}}$ as long as $\pi_{\theta_{i}}$ is feasible, because at least one feasible solution $\pi_{\theta_{i}}$ within the trust region exists. By Theorem \ref{theorem:1}, we know that $q_{i+1}^*$ in the E-step is guaranteed to be feasible and optimal.
\end{proof}

\begin{Remark}
Proposition \ref{Proposition:slater_gurantee} is useful in practice, since it allows the policy to recover to the feasible region from perturbations in M-step. The perturbations may come from bad function approximation errors in M-step or noisy Q estimations, which may happen occasionally for black-box function approximators such as neural network. The robustness guarantee holds even under the worst-case scenario with adversarial perturbations.
\end{Remark}

\textbf{Figure illustrations}.
Fig. \ref{fig:robustness_a} demonstrates the example of one EM iteration that is subject to approximation errors in the M-step. The green area represents the feasible region in the parameter space, the yellow ellipsoid is the trust region in the M-step, and dashed blue circles are the trust region size in the E-step. The righter region has the higher reward in this counter example. Ideally, the updated policy should be the intersection point of the $\pi_{\theta_i}$ -- $q_i^*$ line and the yellow ellipsoid. However, due to the approximation errors in the M-step, the updated policy $\pi_{\theta_{i+1}}$ might be away from the correction updating direction. Fig. \ref{fig:robustness_b} shows the worst-case scenario, where the updating direction is totally orthogonal to the correct one. However, due to the smaller trust region size of M-step than E-step, the Slater condition still holds -- a feasible policy exist within the trust region of $\pi_{\theta_{i+1}}$ that is specified by $\epsilon_2$. So, we could guarantee to obtain an optimal and feasible non-parametric variational distribution at the $(i+1)$-th iteration -- the policy still has the chance to recover to the feasible region.

\begin{figure}[h]
\centering     
\subfigure[Regular policy updating]{\label{fig:robustness_a}\includegraphics[width=0.45\linewidth]{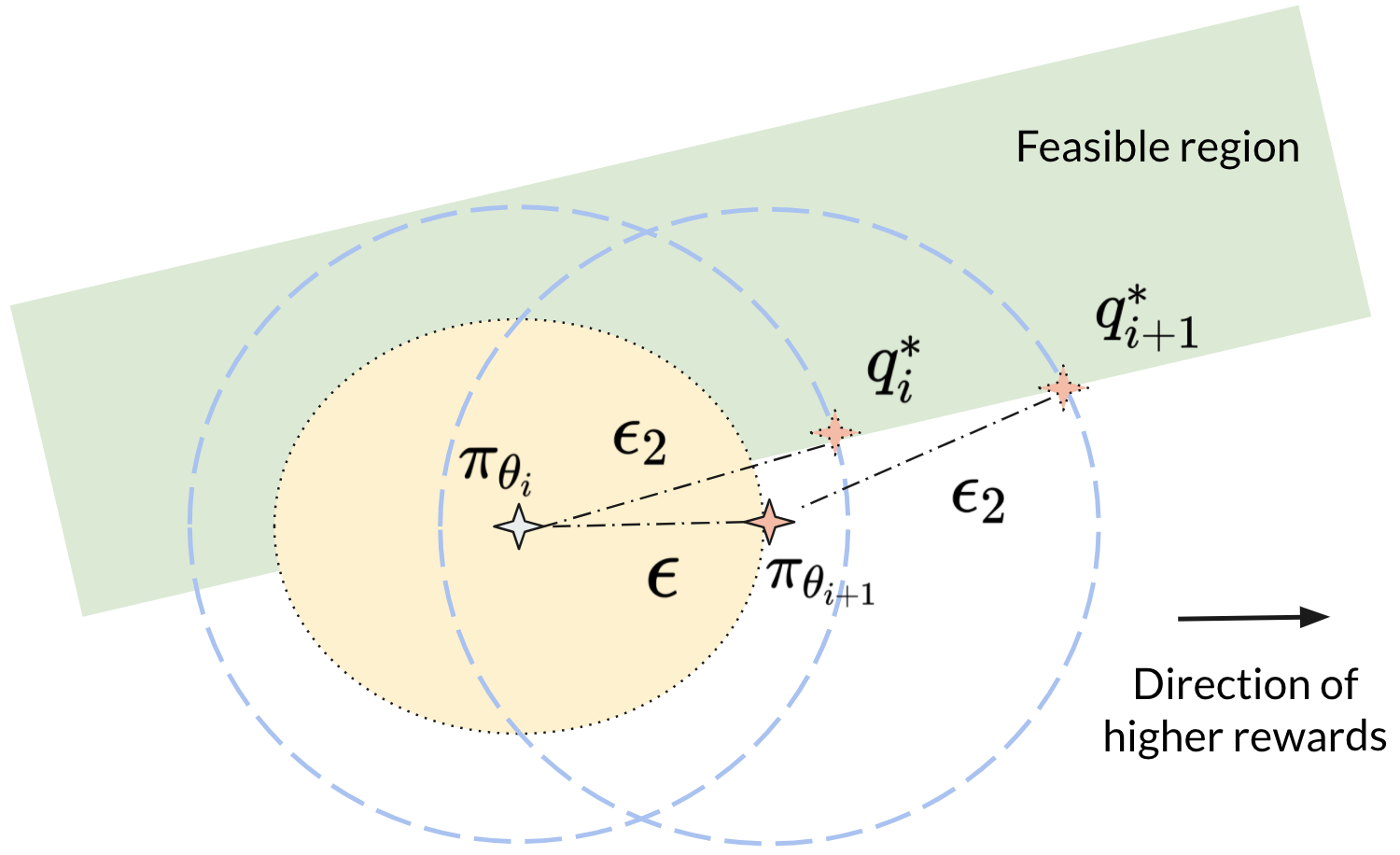}}
\subfigure[Worst-case policy updating]{\label{fig:robustness_b}\includegraphics[width=0.45\linewidth]{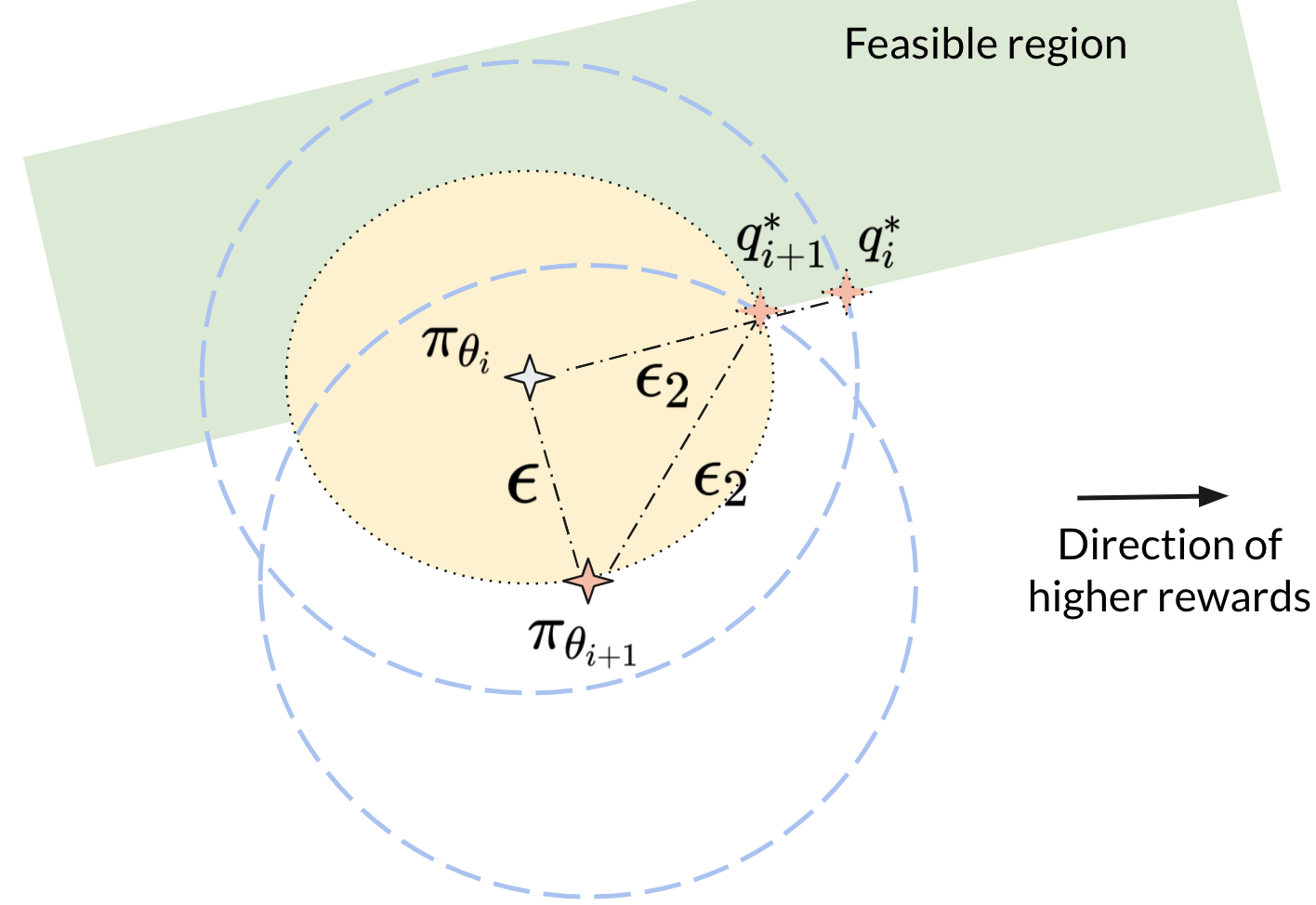}}
\caption{Illustration of the policy updating at the $i$-th iteration under the M-step approximation error.}
\end{figure}

\textbf{Extending the one-step robustness guarantee to multiple steps.} A natural follow-up question for Proposition \ref{Proposition:slater_gurantee} is that whether we can achieve multiple steps policy updating robustness guarantees -- can the policy recover to feasible region with $n>1$ steps adversarial/worst-case M-step policy updating? 
Since the $n$-step ($n>1$) guarantees may require the Triangular inequality to be satisfied for consecutive updated policies, but it may not hold for KL divergence in general cases. However, we found that if the policy $\pi_\theta$ is of a multivariate Gaussian form, which is commonly used for continuous action space tasks in practice, then a relaxed Triangular inequality holds -- see Theorem 4 in \cite{zhang2021properties}.
Thus, we hypothesize that with a Gaussian policy, $n$-step robustness could be achieved with sufficiently small $\epsilon = f(\epsilon_2, n)$ with the function $f$.
Next, we provide a proof for two steps robustness.

\begin{Proposition}
Suppose $\pi_{\theta_{i}} \in \Pi_\Qcal^{\epsilon_1}$ and $\pi$ is a Multivariate Gaussian policy. If $\epsilon < \frac{\epsilon_2}{8}$, where $\epsilon, \epsilon_2$ are the KL threshold in M-step and E-step respectively, then the variational distribution $q_{i+2}^*$ in the $(i+2)$-th iteration is guaranteed to be feasible and optimal.
\end{Proposition}

\begin{proof}
Denote $W_0$ and $W_{-1}$ be the 0, -1 branches of the Lambert $W$ function, respectively. We use $\pi_i$ denote $\pi_{\theta_i}$ for simplicity.
Since $\KL( \pi_{i} \Vert \pi_{i+1}) \leq \epsilon$ and $\KL( \pi_{i+1} \Vert \pi_{i+2}) \leq \epsilon$, we have the following relaxed Triangular inequality holds (Theorem 4 in \cite{zhang2021properties}):
\begin{align}
  \KL( \pi_i \Vert\pi_{i+2}) < 2\epsilon + \frac{1}{2}\Bigg(
  \left(W_{-1}(-e^{(-1-2\epsilon)}) + 1 \right)^2 - W_{-1}(-e^{(-1-2\epsilon)})
  \left(  \sqrt{\frac{2\epsilon}{- W_{0}(-e^{(-1-2\epsilon)})}} + \sqrt{2\epsilon}
  \right)^2
  \Bigg).
\end{align}
Note that $W_0(-1/e) = W_{-1}(-1/e)= -1$ and for sufficiently small $\epsilon$, $W_{-1}(-e^{(-1-2\epsilon)})$ and $W_{0}(-e^{(-1-2\epsilon)})$ are arbitrarily close to $-1$ by the following series \cite{corless1996lambertw}
\begin{align}
W_{-1}(-e^{(-1-2\epsilon)}) = -1 - 2\sqrt{\epsilon} + O(\epsilon); \quad W_{0}(-e^{(-1-2\epsilon)}) = -1 + 2\sqrt{\epsilon} - O(\epsilon).
\end{align}
So we have
\begin{align}
  \left(W_{-1}(-e^{(-1-2\epsilon)}) + 1 \right)^2 = (- 2\sqrt{\epsilon} + O(\epsilon))^2 = 4\epsilon - O(\epsilon^{1.5})
\end{align}


\begin{align}
  - W_{-1}(-e^{(-1-2\epsilon)})
  \left(  \sqrt{\frac{2\epsilon}{- W_{0}(-e^{(-1-2\epsilon)})}} + \sqrt{2\epsilon}
  \right)^2  &= (1+2\sqrt{\epsilon} + O(\epsilon)) \left(  \sqrt{\frac{2\epsilon}{1-2\sqrt{\epsilon} + O(\epsilon)}} + \sqrt{2\epsilon}
  \right)^2 \\
  & \leq (1+2\sqrt{\epsilon} + O(\epsilon)) \left(  \frac{4\epsilon}{1-2\sqrt{\epsilon} + O(\epsilon)} + 4\epsilon
  \right) \\
  & =  8\epsilon + O(\epsilon^{1}) + O(\epsilon^{1.5}).
\end{align}

Then we obtain the following bound by ignoring the high order terms for sufficiently small $\epsilon$:
\begin{align}
  \KL( \pi_i \Vert\pi_{i+2}) < 2\epsilon + \frac{1}{2}(4\epsilon+8\epsilon) = 8\epsilon.
\end{align}
Therefore, we could conclude that for sufficiently small trust-region size, and $\epsilon < \frac{\epsilon_2}{8}$, we could obtain two steps robustness guarantee -- no matter how worse are the M-step for two iterations, an optimal and feasible variational distribution could be solved, and the policy could be recovered to the safe region.
\end{proof}

While the above proof could be generalized to $n$-step robustness, we found that as the certified step $n$ increases, the trust region size $\epsilon$ in M-step has to be shrunk with a faster rate than $n$. In addition, we could observe that though smaller trust region size in M-step could improve the robustness, it will also reduce the training efficiency -- more training steps and samples are required to improve the policy. Therefore, there exists a natural robustness and efficiency trade-off in this context. As we have shown in the experiment section, we believe that one or two steps robustness should be enough to handle mild approximation errors in practice, since the MLE style objective in M-step will converge to the true distribution in probability (consistency) and achieve the lowest-possible variance of parameters (efficiency) asymptotically as the sample size increases~\cite{myung2003tutorial}. This principle also guides us to select reasonable batch size and particle size practically.


\clearpage
\section{Implementation Details}
\label{app:implementation_detail}
\subsection{Full algorithm}
\label{app:implementation}
Due to the page limit, we omit some implementation details in the main content. We will present the full algorithm and some implementation tricks in this section. Without otherwise statement, the critics' and policies'  parametrization is assumed to be neural networks (NN), while we believe other parametrization form should also work well.

\textbf{Critics update}. 
Denote $\phi_r$ as the parameters for the task reward critic $Q_r$, and $\phi_c$ as the parameters for the constraint violation cost critic $Q_c$. Similar to many other off-policy algorithms~\cite{lillicrap2015continuous}, we use a target network for each critic and the polyak smoothing trick to stabilize the training. Other off-policy critics training methods, such as Re-trace~\cite{munos2016safe}, could also be easily incorporated with CVPO training framework. Denote $\phi_r'$ as the parameters for the \textbf{target} reward critic $Q_r'$, and $\phi_c'$ as the parameters for the \textbf{target} cost critic $Q_c'$. Define $\Dcal$ as the replay buffer and $(s, a, s', r, c)$ as the state, action, next state, reward, and cost respectively. The critics are updated by minimizing the following mean-squared Bellman error (MSBE):
\begin{align}
  & L(\phi_r) = \E_{(s, a, s', r, c) \sim \Dcal}\Big[\left( Q_r(s, a) -  (r + \gamma \E_{a'\sim \pi}[ Q_r'(s', a') ] ) \right)^2 \Big]
  \label{eq:qr_loss}\\
  & L(\phi_c) = \E_{(s, a, s', r, c) \sim \Dcal} \Big[\left( Q_c(s, a) -  (c + \gamma \E_{a' \sim \pi}[ Q_c'(s', a') ] ) \right)^2 \Big].
  \label{eq:qc_loss}
\end{align}
Denote $\alpha_c$ as the critics' learning rate, we have the following updating equations:
\begin{align}
    & \phi_r \xleftarrow{} \phi_r - \alpha_c \nabla_{\phi_r} L(\phi_r) \\
    & \phi_c \xleftarrow{} \phi_c - \alpha_c \nabla_{\phi_c} L(\phi_c).
\end{align}

\textbf{M-step regularized policy improvement trick}. We use a Multivariate Gaussian policy in our implementation. As shown in MPO \cite{abdolmaleki2018maximum, abdolmaleki2018relative}, decoupling the KL constraint into two separate terms -- mean $\epsilon_\mu$ and covariance $\epsilon_\Sigma$, can yield better empirical performance. Denote $C_\mu = \E_{\rho_q}\big[ \frac{1}{2}\textbf{tr}(\Sigma^{-1}\Sigma_i)-n + \ln(\frac{\Sigma}{\Sigma_i}) \big]$ and $C_\mu = \E_{\rho_q}\big[ \frac{1}{2} (\mu-\mu_i)^T\Sigma^{-1}(\mu-\mu_i) \big]$, we have:
\begin{equation}
    \E_{\rho_q}\big[ \KL(\pi_{\theta_i}(a|s)\Vert \pi_{\theta}(a|s)) \big] = C_\mu + C_\Sigma.
\end{equation}
And thus the M-step objective could be written as the following Lagrangian function:
\begin{equation}
\begin{split}
    & L(\theta, \beta_\mu, \beta_\Sigma) =  \E_{\rho_q}\Big[ \E_{q^*_i(\cdot|s)} \big[ \log \pi_\theta (a|s)\big]\Big] + \beta_\mu(\epsilon_\mu - C_\mu) +  \beta_\Sigma(\epsilon_\Sigma - C_\Sigma).
\end{split}
\label{eq:m_step_dual_problem}
\end{equation}

By performing the gradient descend ascend algorithm over the dual variables $\beta_\mu, \beta_\Sigma$ and the policy parameters $\theta$ in Eq. (\ref{eq:m_step_dual_problem}) iteratively yields the KL-constrained policy improvement in a supervised learning fashion:
\begin{equation}
\begin{split}
    &\max_\theta \min_{\beta_\mu>0, \beta_\Sigma>0} \quad L(\theta, \beta_\mu, \beta_\Sigma).
\end{split}
\end{equation}
Denote $\alpha_\mu, \alpha_\Sigma, \alpha_\theta$ as the learning rate for $\beta_\mu, \beta_\Sigma, \theta$ respectively, we have the following updating equations:
\begin{align}
    &\beta_\mu \xleftarrow{} \beta_\mu - \alpha_\mu \frac{\partial L(\theta, \beta_\mu, \beta_\Sigma)}{\partial \beta_\mu} = \beta_\mu - \alpha_\mu (\epsilon_\mu - C_\mu) \label{eq:m_step_mean}\\
    &\beta_\Sigma \xleftarrow{} \beta_\Sigma - \alpha_\Sigma \frac{\partial L(\theta, \beta_\mu, \beta_\Sigma)}{\partial \beta_\Sigma} = \beta_\Sigma - \alpha_\Sigma (\epsilon_\Sigma - C_\Sigma) \label{eq:m_step_var}\\
    &\theta \xleftarrow{} \theta - \alpha_\theta \frac{\partial L(\theta, \beta_\mu, \beta_\Sigma)}{\partial \theta}\label{eq:m_step_theta}.
\end{align}

In practice, we also use a target policy network $\pi_{\theta'}$ to generate the covariance matrix $\Sigma$ and the current policy network $\pi_{\theta}$ to generate the mean vector $\mu$, such that the policy will not be easily collapsed to a local optimum. 

\textbf{Polyak averaging for the target networks}. The polyak averaging is specified by a weight parameter $\rho \in (0, 1)$ and updates the parameters with:
\begin{equation}
\begin{aligned}
  &\phi_r' = \rho \phi_r' + (1-\rho) \phi_r \\ &\phi_c' = \rho \phi_c' + (1-\rho) \phi_c \\
  &\theta' = \rho \theta' + (1-\rho) \theta .
\end{aligned}
\label{eq:polyak}
\end{equation}

With all the implementation tricks mentioned above, we present the full CVPO algorithm:

\begin{algorithm}[H]
\caption{CVPO Algorithm}
{\bfseries Input:} \raggedright rollouts $T$, M-step iteration number $M$, batch size $B$, particle size $K$, discount factor $\gamma$, polyak weight $\rho$, critics learning rate $\alpha_c$, policy learning rate $\alpha_\theta$, M-step dual variables' learning rates $\alpha_\mu, \alpha_\Sigma$, thresholds $\epsilon_\mu, \epsilon_\Sigma$ \par
{\bfseries Output:} \raggedright policy $\pi_\theta$ \par
\begin{algorithmic}[1] 
\STATE Initialize policy parameters $\theta, \theta'$, critics parameters $\phi_r, \phi_r', \phi_c, \phi_c'$ and replay buffer $\Dcal = \{\}$
\FOR{each training iteration}
\STATE Rollout $T$ trajectories by $\pi_\theta$ from the environment $\Dcal = \Dcal \cup \{(s, a, s', r, c)\}$

\STATE Sample $B$ transitions $\{(s_b, a_b, s_{b+1}, r_b, c_b)_{b=1,...,B}\}$ from the replay buffer $\Dcal$
\STATE $\triangleright$ \textit{E-step begins}
\STATE Update reward critic by Eq. (\ref{eq:qr_loss}): $\phi_r \xleftarrow{} \phi_r - \alpha_c \nabla_{\phi_r} L(\phi_r) $
\STATE Update cost critic by Eq. (\ref{eq:qc_loss}): $\phi_c \xleftarrow{} \phi_c - \alpha_c \nabla_{\phi_c} L(\phi_c) $
\FOR{$b=1,..., B$}
\STATE Sample $K$ actions $\{a_1,...,a_K\}$ for $s_b$
\STATE Compute $\{Q_{r}^{\theta_i}(s_b,a_k), Q_{c}^{\theta_i}(s_b,a_k); k=1,...,K\}$
\ENDFOR
\STATE Compute optimal dual variables $\eta^*, \lambda^*$ by solving the convex optimization problem (\ref{eq:dual_opt})
\STATE Compute the optimal variational distribution for each state $\{q^*(\cdot|s_b); b=1,...,B\}$ by Eq. (\ref{eq:optimalq_theorem})
\STATE Normalize the variational distribution $\{q^*(\cdot|s_b); b=1,...,B\}$ for each state
\STATE $\triangleright$ \textit{M-step begins}
\FOR{M-step iterations $m=1,..., M$}
\STATE Perform one gradient step for $\beta_\mu$ via Eq. (\ref{eq:m_step_mean}) and for $\beta_\Sigma$ via Eq. (\ref{eq:m_step_var})
\STATE Perform one gradient step for policy parameters via Eq. (\ref{eq:m_step_theta}): $\theta \xleftarrow{} \theta - \alpha_\theta \frac{\partial L(\theta, \beta_\mu, \beta_\Sigma)}{\partial \theta}$
\ENDFOR
\STATE Polyak averaging target networks by Eq. (\ref{eq:polyak})
\ENDFOR
\end{algorithmic}
\end{algorithm}

Note that for off-policy methods, we need to convert the episodic-wise constraint violation threshold to a state-wise threshold for the $Q_c$ functions.
Denote $T$ as the episode length, the target cost limit for one episode is $\epsilon_T$. Denote the discounting factor as $\gamma$. Then, if we assume that at each time step we have equal probability to violate the constraint, the target constraint value $\epsilon_{1}$ for safety critic $Q_c^{\pi_\theta}$ could be approximated by:
$$
\epsilon_{1} =  \epsilon_T \times \frac{1-\gamma^T}{T(1-\gamma)}
$$
The converted threshold $\epsilon_1$ will be used to compute the Lagrangian multipliers for the baselines, and also be used as one of the constraint threshold in the E-step of our method:
$$
\int \pi(a|s) Q_c^{\pi_{\theta_i}}(s,a) \leq \epsilon_{1}, \quad\forall s, a
$$

\subsection{Experiment environments}
The task environment implementations are built upon SafetyGym (based on Mujoco)~\cite{ray2019benchmarking} and its PyBullet implementation~\cite{gronauer2022bullet}. 
We modified the original environment parameters for all the safe RL algorithms to make the training faster and save computational resources.
Particularly, we increase the simulation time-step and decrease the timeout steps for each environment, such that the agent can finish the tasks with fewer steps.
In addition, the Goal task in this paper is modified from the Button task, since we can then fix the layout of the goal buttons and obstacles to make the environment more deterministic. 
Note that the original SafetyGym implementation will random sample the layout for each episode, which greatly increase the training time and variance.
The proposed CVPO implementation also works for the original SafetyGym environments, but may require different set of hyper-parameters and much longer training time.
Though CVPO is sample-efficient, it is not very computational efficient, since we need to sample many particles in the E-step and solve a convex optimization problem, which is currently done on CPU via SciPy.

\subsection{Hyper-parameters}
\label{app:hyper-param}

The hyperparameters are shown in Table \ref{table:hyper-param}. More details can be found in the code.

\begin{table}[h]
\centering
\begin{tabular}{ cc }   
Common Hyperparameters & CVPO Hyperparameter \\  
\begin{tabular}{|c|c|}
\hline 
Policy network sizes & [256, 256] \\ \hline
Q network sizes & [256, 256] \\ \hline
Network activation & ReLU \\ \hline
Discount factor gamma $\gamma$ & 0.99 \\ \hline
Polyak weight $\rho$: & 0.995 \\ \hline
Batch size $B$: & 300 \\ \hline
Rollout trajectory number $T$ & 20 \\ \hline
Critics learning rate $\alpha_c$ & 0.001 \\ \hline
NN Optimizer & Adam \\ \hline
\end{tabular} &  
\begin{tabular}{|c|c|}
\hline 
Particle size $K$ & 32 \\ \hline
M-step iterations $M$ & 6 \\ \hline
Learning rate $\alpha_\mu$ & 1\\ \hline
Learning rate $\alpha_\Sigma$ & 100\\ \hline
Learning rate $\alpha_\theta$ & 0.002\\ \hline
E-step KL threshold $\epsilon_2$: & 0.1 \\ \hline
M-step KL threshold $\epsilon_\mu$: & 0.001 \\ \hline
M-step KL threshold $\epsilon_\Sigma$: & 0.0001 \\ \hline
E-step solver & SLSQP \\ \hline
\end{tabular} \\
\end{tabular}
\caption{Hyperparameters. Left: common hyperparameters for all methods. Right: hyperparameters that are specifically for CVPO.}
\label{table:hyper-param}
\end{table}

\section{Complete Experiment Results}
\label{app:experiment_result}

We present all the experiment results in this section.

\begin{figure}[!htb]
\centering     
\includegraphics[width=0.9\linewidth]{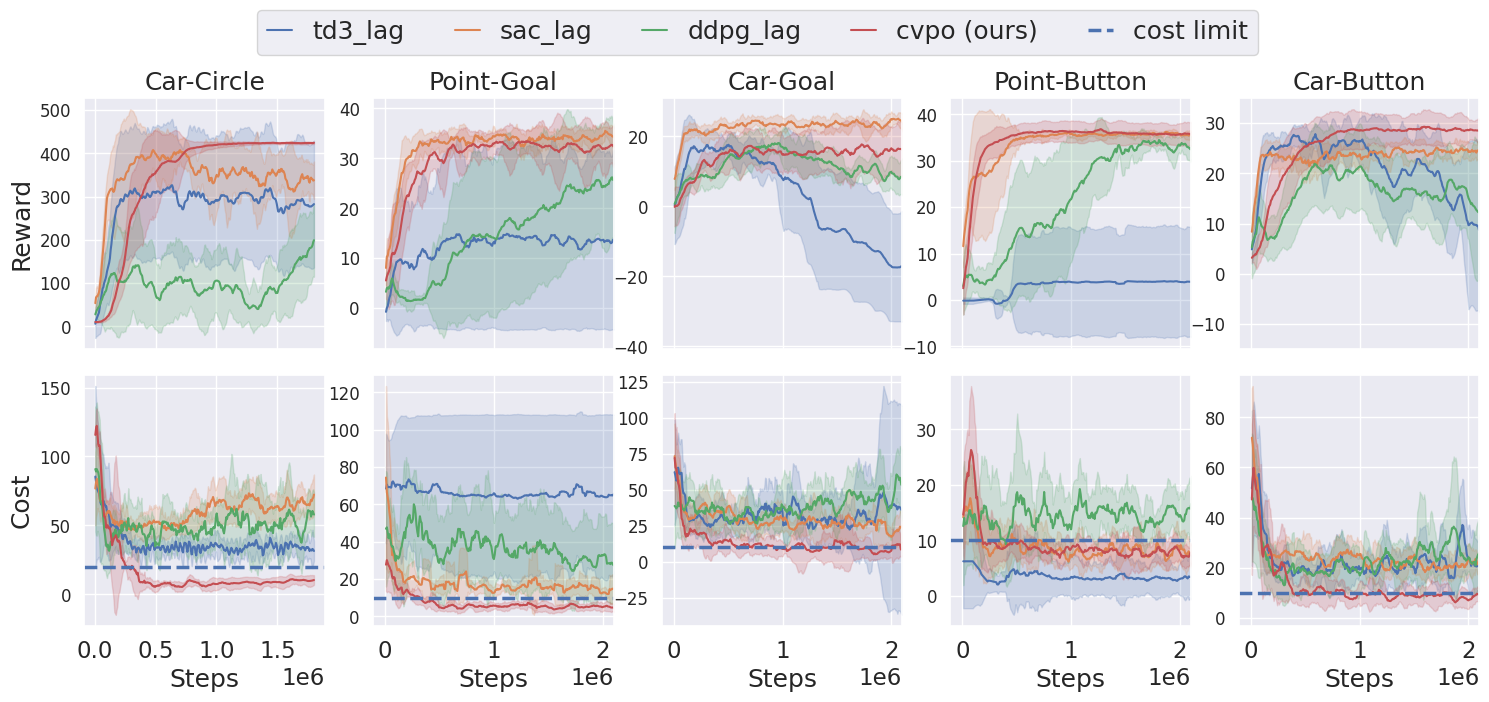}
\caption{Training curves for off-policy baselines comparison. Each column corresponds to an environment. The curves are averaged over 10 random seeds, where the solid lines are the mean and the shadowed areas are the standard deviation.}
\end{figure}

\begin{figure}[!htb]
\centering     
\includegraphics[width=0.9\linewidth]{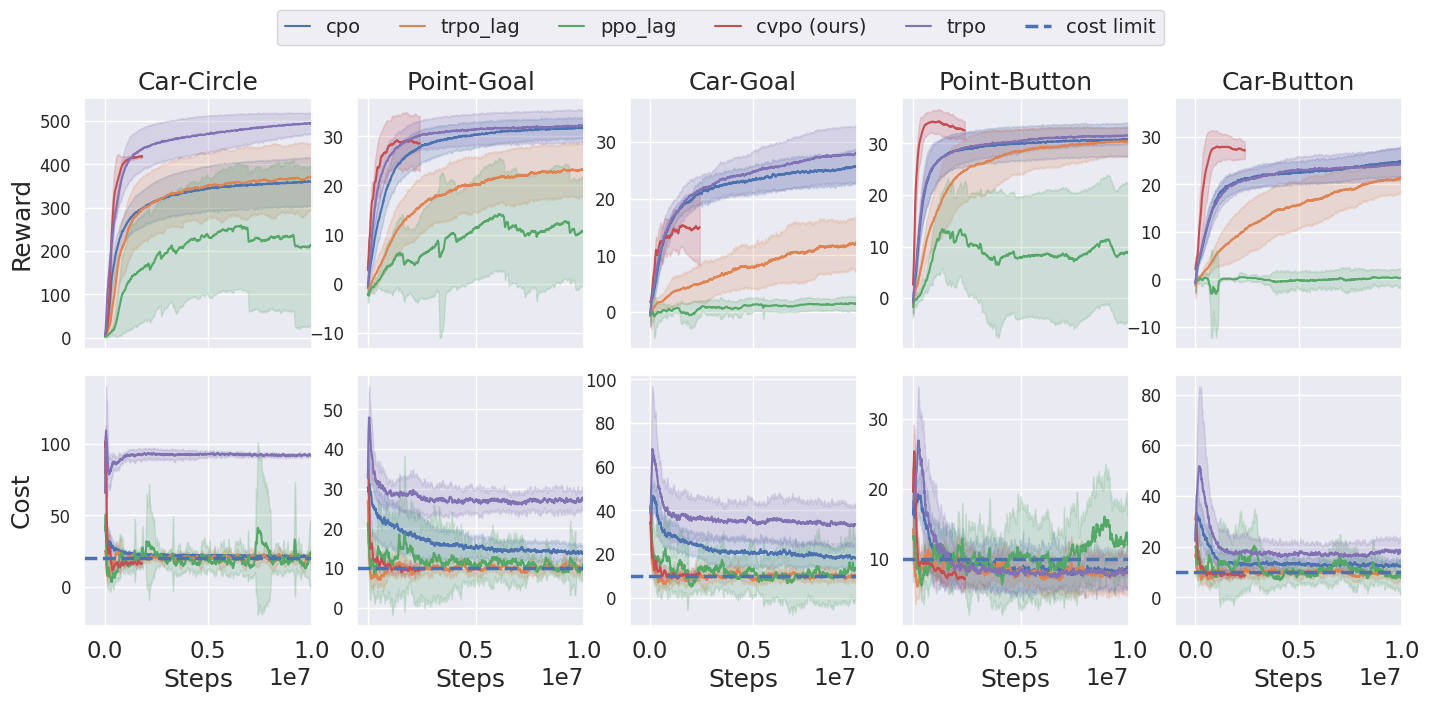}
\caption{Training curves for on-policy baselines comparison. Each column corresponds to an environment. The curves are averaged over 10 random seeds, where the solid lines are the mean and the shadowed areas are the standard deviation.}
\end{figure}

\begin{figure}[!h]
\centering     
\includegraphics[width=0.95\linewidth]{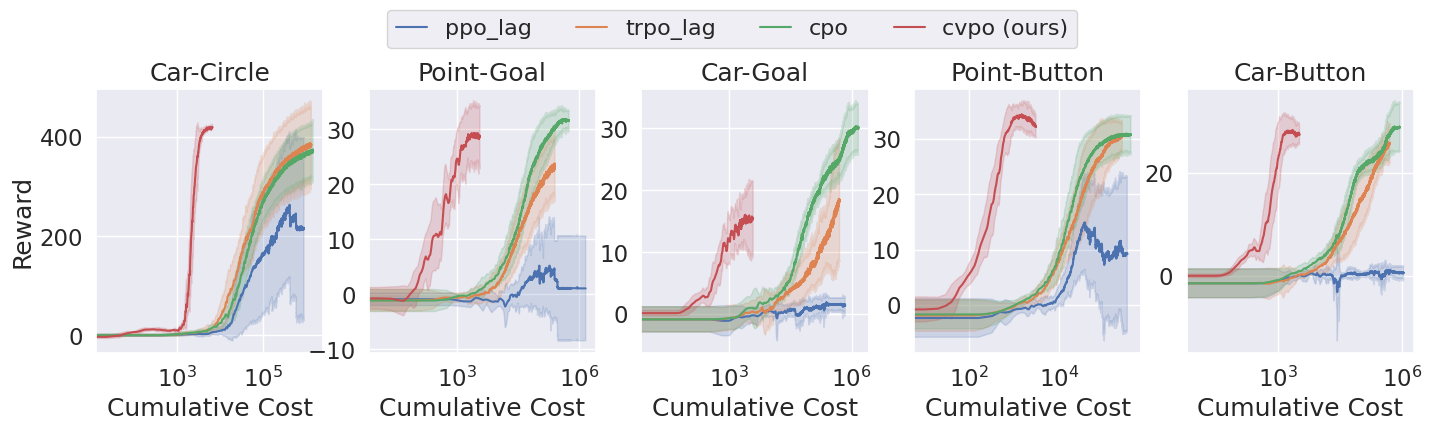}
\caption{Reward versus cumulative cost (log-scale).}
\label{fig:reward_cost_all}
\end{figure}

\begin{figure}[h]
\centering     
\subfigure[]{\label{fig:robustness_a_all}\includegraphics[width=0.45\linewidth]{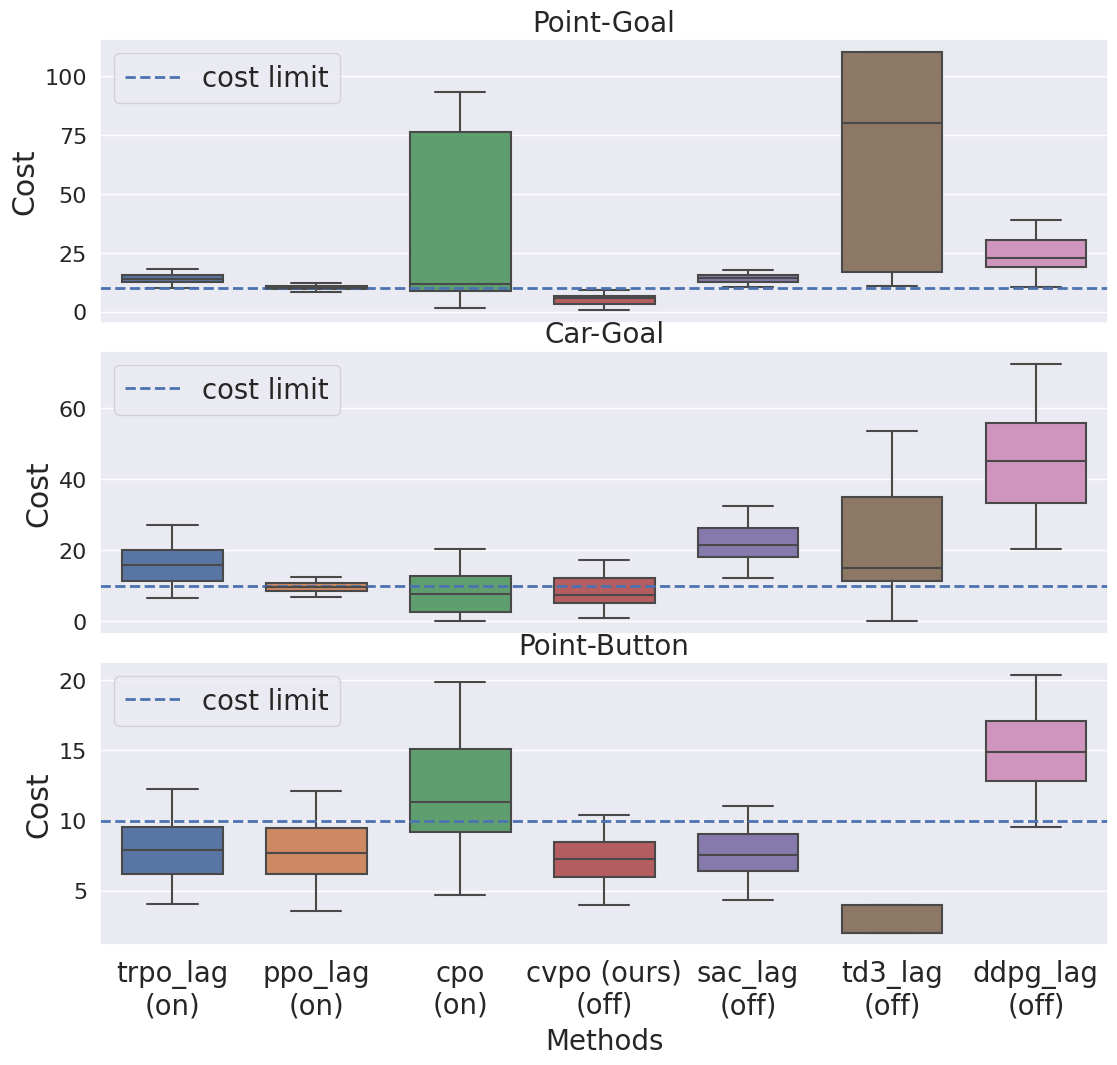}}
\subfigure[]{\label{fig:robustness_b_all}\includegraphics[width=0.45\linewidth]{figures/boxplot.png}}
\caption{Box plot of the convergence cost. (on) and (off) denotes on-policy and off-policy method, respectively.}
\end{figure}

\begin{figure}[!h]
\centering     
\includegraphics[width=0.95\linewidth]{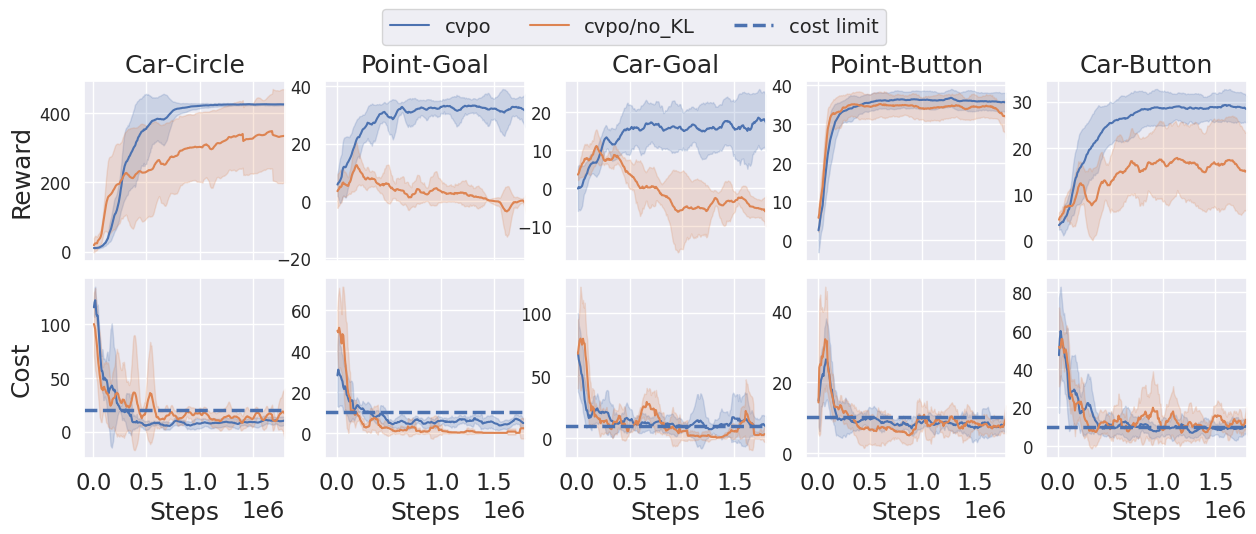}
\caption{Ablation study of the KL constraint in the M-step.}
\label{fig:kl_compare_all}
\end{figure}




\end{document}